\documentclass
[a4paper,11pt]{scrartcl}
\pdfoutput=1 
\usepackage[utf8]{inputenc}
\usepackage{microtype}
\usepackage[numbers]{natbib}

\title{Destructiveness of Lexicographic Parsimony Pressure and Alleviation by a Concatenation Crossover in Genetic Programming}

\author{Timo Kötzing$^\dagger$ \and J.~A.~Gregor Lagodzinski$^\dagger$ \and Johannes Lengler$^\star$ \and Anna Melnichenko$^\dagger$}
\publishers{$\mbox{}$ \\ \small $\dagger:$ Hasso Plattner Institute, University of Potsdam, Potsdam, Germany \\
			$\star:$ ETH Zürich, Zürich, Switzerland}
\date{\today}

\newcommand{\order}{\textsc{Order}\xspace}
\newcommand{\majority}{\textsc{Majority}\xspace}
\newcommand{\supermajority}{{\ensuremath{2/3}-\textsc{SuperMajority}}\xspace}
\newcommand{\geqcmajority}{{\ensuremath{+c}-\textsc{Majority}}\xspace}
\newcommand{\geqfracmajority}{{\ensuremath{2/3}-\textsc{Majority}}\xspace}

\newcommand{\tinit}{s_{\mathrm{init}}}
\newcommand{\sinit}{s_{\mathrm{init}}}
\newcommand{\gptree}{GP-tree\xspace}
\newcommand{\gptrees}{GP-trees\xspace}

\newcommand{\var}{x}
\newcommand{\nonvar}{\overline{x}}

\newcommand{\EE}{\mathbb{E}}
\newcommand{\eps}{\varepsilon}

\newcommand{\BigO}{\mathrm{O}}
\newcommand{\LittleO}{\mathrm{o}}

\newcommand{\Prob}[1]{\mathrm{Pr}\left[#1 \right]}
\newcommand{\Ex}[1]{\mathbb{E}[#1 ]}

\newcommand{\event}[1]{\ensuremath{\mathcal{{#1}}}}

\usepackage{xspace}
\usepackage{amsfonts,amssymb,amstext,amsthm,amsmath}
\usepackage{mathabx}
\usepackage[mathscr]{euscript}
\usepackage{xcolor}
\usepackage{caption}
\usepackage{calc}
\usepackage{booktabs}
\usepackage{graphicx}
\usepackage{enumerate}
\usepackage{mathtools}
\usepackage{lmodern}
\usepackage{datetime}
\usepackage{listings}
\usepackage[english]{babel}
\usepackage{natbib}
\usepackage{tabularx}

\usepackage{pgfplots}
\pgfplotsset{compat=1.10}
\usepgfplotslibrary{statistics}

\theoremstyle{plain}
\newtheorem{theorem}{Theorem}[section]
\newtheorem{corollary}[theorem]{Corollary}
\newtheorem{lemma}[theorem]{Lemma}

\theoremstyle{definition}

\usepackage[ruled,vlined,linesnumbered,algo2e]{algorithm2e}
\SetKwFor{Function}{function}{is}{end function}
\newcommand{\assign}{\leftarrow}
\SetKw{Input}{input}
\SetKw{Output}{output}
\SetKw{Initialization}{initialization}

\DeclareOldFontCommand{\bf}{\normalfont\bfseries}{\mathbf}

\newcommand{\oneonegp}{(1+1)~GP\xspace}

\newcommand{\N}{\mathbb{N}}

\newcommand{\R}{\mathbb{R}}

\usepackage{tikz}
\def\checkmark{\tikz\fill[scale=0.4](0,.35) -- (.25,0) -- (1,.7) -- (.25,.15) -- cycle;} 

\begin{document}
\maketitle

\begin{abstract}
For theoretical analyses there are two specifics distinguishing GP from many other areas of evolutionary computation. First, the variable size representations, in particular yielding a possible bloat (i.e.\ the growth of individuals with redundant parts). Second, the role and realization of crossover, which is particularly central in GP due to the tree-based representation. Whereas some theoretical work on GP has studied the effects of bloat, crossover had a surprisingly little share in this work.

We analyze a simple crossover operator in combination with local search, where a preference for small solutions minimizes bloat (\emph{lexicographic parsimony pressure}); the resulting algorithm is denoted \emph{Concatenation Crossover GP}. For this purpose three variants of the well-studied \majority test function with large plateaus are considered. We show that the Concatenation Crossover GP can efficiently optimize these test functions, while local search cannot be efficient for all three variants independent of employing bloat control.
\end{abstract}

\section{Introduction}
\label{sec:intro}
Genetic Programming (GP) is a field of Evolutionary Computing (EC) where the evolved objects encode programs. Usually a tree-based representation of a program is iteratively improved by applying variation operators (mutation and crossover) and selection of suitable offspring according to their quality (fitness). Most other areas of EC deal with fixed-length representations, whereas the tree-based representation distinguishes GP. This representation of variable size leads to one of the main problems when applying GP: \emph{bloat}, which describes an unnecessary growth of representations. Solutions may have many redundant parts, which could be removed without afflicting the quality, and search is slowed down, wasted on uninteresting areas of the search space.

In this paper we study GP from the point of view of run time analysis. While many previous theoretical works analyzed mutational GP with the offspring produced by varying a single parent, we analyze a GP algorithm employing a simple crossover with the offspring produced from two parents. Although our crossover is far from practical applications of GP (it merely concatenates the two parent trees), this simple setting aims at understanding the interplay between (our variant of) crossover, the problem of bloat and \emph{lexicographic parsimony pressure}, a method for bloat control introduced in \cite{Luke:2002:GECCO}. Other theoretical work in GP has analyzed different problems and phenomena, in particular for the Probably Approximately Correct (PAC) learning framework~\cite{DBLP:conf/gecco/KotzingNS11}, the Max-Problem~\cite{Gathercole96,Langdon97ananalysis,KoeSutNeuOre:c:12} as well as Boolean functions~\cite{MoraglioMM13,MambriniM14,MambriniO16}.

For the effects of bloat in the sense of redundant parts in the tree, we draw on previous theoretical works that analyzed this phenomenon, especially~\cite{NeuGECCO12} and~\cite{doerr2017bounding}. In these, the fitness function \majority as introduced in \cite{GoldbergO98} was analyzed. Individuals for \majority are binary trees, where each inner node is labeled $J$ (short for \emph{join}, but without any associated semantics) and leaves are labeled with variable symbols; we call such trees \emph{\gptrees}. The set of variable symbols is  $\{\var_1,\ldots,\var_n\} \cup \{\nonvar_1,\ldots,\nonvar_n\}$, for some $n$. In particular, variable symbols are paired: $\var_i$ is paired with $\nonvar_i$. For \majority, we call a variable symbol $\var_i$ \emph{expressed} if there is a leaf labeled $\var_i$ and there are at least as many leaves labeled $\var_i$ as there are leaves labeled $\nonvar_i$; the positive instances are in the majority. The fitness of a \gptree is the number of its expressed variable symbols $\var_i$. 
See Figure~\ref{fig:tree_examples} for two exemplary \gptrees. 
\begin{figure}[t]
	\begin{minipage}{0.4\textwidth}
		\centering
		\includegraphics[width = 0.9\textwidth]{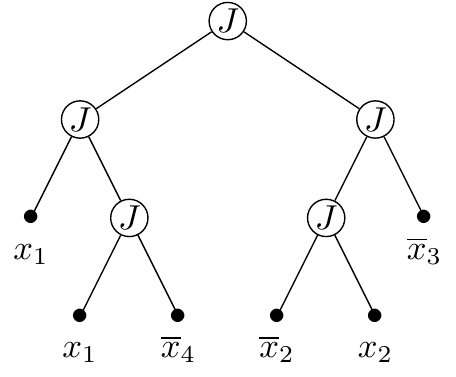}
	\end{minipage} \hfill
	\begin{minipage}{0.4\textwidth}
		\centering
		\includegraphics[width = 0.8\textwidth]{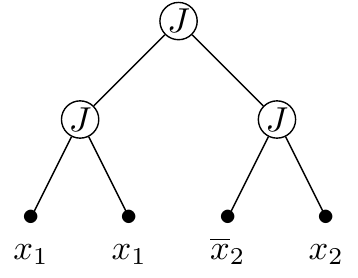}
	\end{minipage}
	\caption{Two GP-trees with the same fitness. For \geqfracmajority the fitness is $2$ and for \geqcmajority with $c=2$ the fitness is $1$. However, the left one has size $6$ whereas the right one has size $4$.}
	\label{fig:tree_examples}
\end{figure}
This setting captures two important aspects of GP: variable length representations and that any given functionality can be achieved by many different representations. However, the tree-structure, typically crucial in GP problems, is completely unimportant for the \majority function.

\begin{table*}[t]%
	\begin{center}
		\caption{Overview of the results of the paper. A check mark denotes optimization in polynomial time with high probability, a cross denotes superpolynomial optimization time. A check mark with  a subscript $e$ denotes the results obtained experimentally.}
		\label{table:overviewResults}
		\renewcommand{\arraystretch}{1.5}
		\begin{tabularx}{\textwidth}{l|X|X|X}
			& \multicolumn{2}{c|}{\textbf{local search}} & \textbf{crossover} \\
			\textbf{problem class} & \textbf{w/ bloat control} & \textbf{w/o bloat control} & \textbf{w/ bloat control}  \\
			\hline
			\geqcmajority  & $\times$~~Theorem~\ref{thm:+cMajority_lower-bound_constant-stepsize} & $\checkmark_e$~~Figure~\ref{fig:2maj_plot} & \checkmark~~Theorem~\ref{thm:crossover}  \\
			\hline
			\geqfracmajority  & \checkmark~~Theorem~\ref{thm: 2-3Majority_bc_constant-stepsize} & $\checkmark_e$~~Figure~\ref{fig:2maj_plot} & \checkmark~~Theorem~\ref{thm:crossover}  \\
			\hline
			\supermajority  & \checkmark~~Theorem~\ref{thm: 2-3SuperMajority_bc_constant-stepsize} & $\times$~~Theorem~\ref{thm:geqfracMajority_lower-bound_constant-stepsize} & \checkmark~~Theorem~\ref{thm:crossoverSupermajority} \\
			\hline
		\end{tabularx}
	\end{center}
\end{table*}

We know that \majority can be efficiently optimized by a mutational GP called \oneonegp (see Algorithm~\ref{alg:gp} for details, basically performing a randomized local search). This holds in the case preferring shorter representations by lexicographic parsimony pressure, as shown in \cite{NeuGECCO12}, as well as in the case without such preference \cite{doerr2017bounding}. Similar to recent literature on theory of GP, we will consider lexicographic parsimony pressure as our method of bloat control and henceforth only speak of \emph{bloat control} to denote this method. We note, however, that the GP literature knows many more methods for controlling bloat which is beyond the scope of our theoretical analysis.

In addition to weighted versions of \majority, another, similar fitness function \order (see also \cite{GPFOGA11,NguUrlWag:c:13:Foga}) has been considered, but neither of these provide us with a strong differences in the optimization behavior of different GP algorithms. Thus, we propose three variants of \majority, called \geqcmajority, \geqfracmajority and \supermajority, which negatively affect the optimization of certain GP algorithms.

For \geqcmajority a variable is expressed if its positive literals are not only in the majority, but also there has to be at least $c$ more positive than negative literals. On the one hand, we show that a random \gptree with a linear number of leaves expresses any given variable with constant probability. On the other hand, with constant probability such a tree has a majority of negative literals of any given variable (indeed, there is a constant probability that the variable has neither positive nor negative literals in the \gptree). This yields a plateau of equal fitness which can only be overcome by adding $c$ positive literals, i.e., we need a rich set of neutral mutations that allow genetic drift to happen. Bloat control suppresses this genetic drift by biasing the search towards smaller solutions. Specifically, it may not allow to add positive literals one by one, which results in an infinite run time (see Lemma~\ref{thm:+cMajority_lower-bound_constant-stepsize}).
Note that allowing the local search to add $c$ leaves at the same time still results only in a small chance of $\BigO(n^{-c})$ of jumping the plateau. Hence, the \geqcmajority fitness function serves as an example where bloat control explicitly harms the search.

For \geqfracmajority, a variable is expressed if its positive literals hold a $2/3$ majority, i.e., if $2/3$ of all its literals are positive. The fitness associated with \geqfracmajority is the number of expressed variables while for \supermajority each expressed variable contributes a score between $1$ and $2$, where larger majorities give larger scores (see Section~\ref{sec:preliminaries} for details). The variant \supermajority is utilized to aggravate the effect of bloat since it rewards large numbers of (positive) literals. We show that local search with bloat control is efficient for these two problems (Theorems~\ref{thm: 2-3Majority_bc_constant-stepsize} and~\ref{thm: 2-3SuperMajority_bc_constant-stepsize}). However, without bloat control local search fails on \supermajority due to bloat (see Theorem~\ref{thm:geqfracMajority_lower-bound_constant-stepsize}).

Regarding optimization without bloat control, we obtain experimental results as depicted in Figure~\ref{fig:2maj_plot}. They provide a strong indicator that, when no bloat control is applied, optimization of \geqcmajority is efficient, in contrast to the case of bloat control. The trend for \geqfracmajority indicates that optimization proceeds significantly more slowly without bloat control than with bloat control. Nevertheless, optimization seems to be feasible in contrast to the case of \supermajority.

Subsequently, we study a simple crossover which works as follows. The algorithm maintains a population of $\lambda$ individuals, which are initialized randomly before a local search with bloat control is performed for a number of iterations. As a local search we employ the \oneonegp, a simple mutation-only GP which iteratively either adds, deletes, or substitutes a vertex of the tree. We employ this algorithm for a number of rounds large enough to ensure that each vertex of the tree has been considered for deletion at least once with high probability, which aims at controlling bloat.
Afterwards, the optimization proceeds in rounds; in each round, each individual $t_0$ is mated with a random other individual $t_1$ by joining $t_0$ and $t_1$ to obtain a tree $t'$ which contains both $t_0$ and $t_1$ (see Figure~\ref{fig:crossover});
\begin{figure}[t]
	\centering
	\includegraphics[width = 0.8\textwidth]{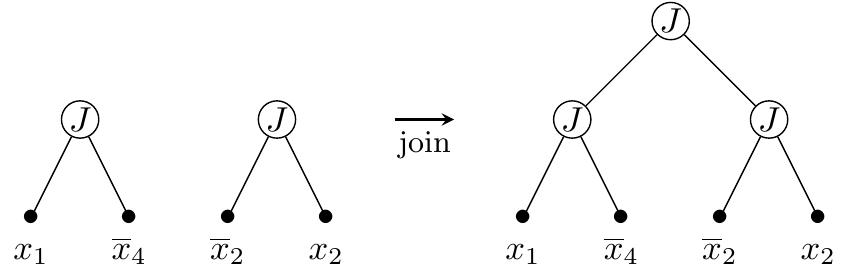}
	\caption{Two GP-trees are joined as in Algorithm~\ref{alg:mimeticCrossover}. A new join-node is introduced and connected to both trees.}
	\label{fig:crossover}
\end{figure}
then local search is performed on $t'$ as before yielding a tree $t''$. If $t''$ is at least as fit as $t_0$, we replace $t_0$ in the population by $t''$. The algorithm is called \emph{Concatenation} since it joins two individuals, which is basically a concatenation. It is different from other approaches for memetic crossover GP as found, for example, in~\cite{Esk-Hou:c:04:MemeticCrossover}. Note that this crossover is very different from GP crossovers found in the literature because of its almost complete disregard for the tree structure of the individuals. However, this crossover already highlights some benefits which can be obtained with crossover, and it has the great advantage of being analyzable.

We show that the Concatenation Crossover GP with bloat control efficiently optimizes all three test functions \geqcmajority, \geqfracmajority as well as \supermajority, due to its ability to combine good solutions (see Theorem~\ref{thm:crossover}). We summarize our findings in Table~\ref{table:overviewResults}.

In Section~\ref{sec:preliminaries} we state the formal definitions of algorithms and problems, as well as the mathematical tools we use. Section~\ref{sec:bloatcontrol} gives the results for local search \emph{with} bloat control, Section~\ref{sec:bloat} for local search \emph{without} bloat control and Section~\ref{sec:crossover} for the Concatenation Crossover GP. In Section~\ref{sec:experiments} we show and discuss our experimental results, before Section~\ref{sec:conclusion} concludes the paper.

This paper is an extended version of a paper which is to appear in the proceedings of PPSN~2018.

\section{Preliminaries}
\label{sec:preliminaries}
For a given $n$ we let $[n] = \{1,\ldots,n\}$ be the set of variables. The only non-terminal (function symbol) is $J$ of arity $2$; the \emph{terminal set} $X$ consists of $2n$ literals, where $\nonvar_i$ is the complement of $\var_i$:
\begin{center}
$F \coloneqq \{J\}$, $J$ has arity $2$, \quad
$X\coloneqq \{\var_1, \nonvar_1, \dots , \var_n, \nonvar_n \}$.
\end{center}

For a \gptree $t$, we denote by $S(t)$ the set of leaves in $t$. By $S_i^+(t)$ and $S_i^-(t)$ we denote the set of leaves that are $\var_i$-literals and $\nonvar_i$-literals, respectively, and by $S_i(t) \coloneqq  S_i^+(t) \cup S_i^-(t)$ we denote the set of all $i$-literals. By $S^+(t) \coloneqq  \bigcup_{i=1}^n S_i^+(t)$ and $S^-(t) \coloneqq  \bigcup_{i=1}^n S_i^-(t)$ we denote the set of all positive and negative leaves, respectively. We denote the sizes of all these sets by the corresponding lower case letters, i.e., $s(t) \coloneqq |S(t)|$, $s_i(t) \coloneqq  |S_i(t)|$, etc.. In particular, we refer to $s(t)$ as the \emph{size} of $t$.

On the syntax trees, we analyze the problems \geqcmajority, \geqfracmajority, and \supermajority, 
which are defined as
\begin{align*}\label{eq:def-of-fitness-functions}
\text{\geqcmajority} & \coloneqq  |\{i \in [n] \mid s_i^+ \geq s_i^- +c\}| \; ;\\
\text{\geqfracmajority} & \coloneqq  |\{i \in [n] \mid s_i \geq 1 \text{ and } s_i^+ \geq \tfrac23 s_i\}| \; ;\\
\text{\supermajority} & \coloneqq  \sum_{i=1}^n f_i, 
\text{ where } f_i \coloneqq 
\begin{cases}
0 & ,\text{if } s_i = 0 \text{ or } s_i^+ < \tfrac23 s_i,\\
2-2^{s_i^--s_i^+} & ,\text{otherwise}.
\end{cases}
\end{align*} 

We call a variable contributing to the fitness \emph{expressed}. Since both \geqcmajority and \geqfracmajority count the number of expressed variables, they take values between $0$ and $n$. The function \supermajority is similar to \geqfracmajority, but if a $2/3$ majority is reached \supermajority awards a bonus for larger majorities: the term $f_i$ grows with the difference $s_i^+ - s_i^-$. Since $f_i \leq 2$, the function \supermajority takes values in $[0,2n]$. Note that the value $2n$ can never actually be reached, but can be arbitrarily well approximated.

In this paper we consider simple mutation-based genetic programming algorithms which use a modified version of the Hierarchical Variable Length (HVL) operator ({\cite{OReilly:thesis}}, \cite{OReilly:1994:GPSAHC}) called HVL-Prime as discussed in~\cite{GPFOGA11}. HVL-Prime allows trees of variable length to be produced by applying three different operations: insert, delete and substitute (see Figure~\ref{fig:HVL}). Each application of HVL-Prime chooses one of these three operations uniformly at random. We note that the literature also contains variants of the mutation operator that apply several such operations simultaneously (see~\cite{GPFOGA11,NguUrlWag:c:13:Foga}).

\begin{figure*}[t]
	{\renewcommand{\arraystretch}{1.5}
	\begin{tabularx}{\textwidth}{|l X|}
		\hline 
		\multicolumn{2}{|>{\hsize=\dimexpr1\hsize+10\tabcolsep}X|}{\centering Given a GP-tree $t$, mutate $t$ by applying HVL-Prime. For each application, choose uniformly at random one of the following three options.}\\ 
		substitute	& Choose a leaf uniformly at random and substitute it with a leaf in $X$ selected uniformly at random. \\ 
		insert  & Choose a node $v \in X$ and a leaf $u \in t$ uniformly at random. Substitute $u$ with a join node $J$, whose children are $u$ and $v$, with the order of the children chosen uniformly at random. \\ 
		delete & Choose a leaf $u \in t$ uniformly at random. Let $v$ be the sibling of $u$. Delete $u$ and $v$ and substitute their parent $J$ by $v$.  \\ 
		\hline 
	\end{tabularx}
	}
	\caption{Mutation operator HVL-Prime.}
	\label{fig:HVL}
\end{figure*}

The first algorithm we study is the \oneonegp. The algorithm is initialized with a tree generated by $\tinit$ random insertions. Afterwards, it maintains the best-so-far individual $t$. In each round, it creates an offspring of $t$ by mutation. This offspring is discarded if its fitness is worse than $t$, otherwise it replaces $t$. We recall that the fitness in the case with bloat control contains the size as a second order term. Algorithm~\ref{alg:gp} states the \oneonegp more formally.

\begin{algorithm2e}
	Let $t$ be a random initial tree of size $\tinit$\;
	\While{optimum not reached}{%
		$t' \assign \mbox{mutate}(t)$\;
		\lIf{$f(t') \geq f(t)$}{$t \assign t'$}
	}
	\caption{\oneonegp with mutations according to Figure~\ref{fig:HVL}}
	\label{alg:gp}
\end{algorithm2e}

\subsection{Crossover}
The second algorithm we consider is population-based.
When introduced by J.~R.~Koza \cite{Koza:1990:GPP}, Genetic Programming used fitness-proportionate selection and a genetic crossover, however mutation was hardly considered. In subsequent works many different setups for the crossover operator were introduced and studied. For instance, in \cite{OReilly:thesis} combinations of GP with local search in the form of mutation operators were studied and yielded better performance than GP.

Usually, two parents (a \emph{current solution} and a \emph{mate}) are used to generate a number of offspring. These offspring are a recombination of the alleles from both parents derived in a probabilistic manner. By modeling each individual as a \gptree, a crossover-point in both parents is decided upon due to a heuristic and the subtrees attached to these points are exchanged creating new \gptrees.

In the Crossover hill climbing algorithm first described by T.~Jones \cite{{jones:1995:ICGA}, {jones:1995:thesis}} only one \gptree is created from the current solution and a random mate. This offspring is evaluated and replaces the current solution if the fitness is not worse.

We consider the following simple crossover: the \emph{Concatenation Cros\-sover GP} working as follows (see also Algorithm~\ref{alg:mimeticCrossover}). For a fixed population of \gptrees, each \gptree is chosen to be the parent once. For each parent we choose a mate uniformly at random from the population and create one offspring by \emph{joining} the two trees using a new join-node. Before evaluating the offspring, we employ a local search in the form of the \oneonegp with bloat control. This local search is performed for a fixed amount of iterations before we discard the \gptree with worse fitness. The fixed amount depends on the size of the tree and ensures the absence of redundant leaves with high probability (see Lemma~\ref{lem:crossover_time_until_minimal}). We note that the amount of redundant leaves depends on the function to be optimized. The functions we studied are variants of \majority, for other functions the amount of iterations ensuring the absence of redundant leaves might be different.  

The initial population is generated by creating $\lambda$ random trees of size $\tinit$ and employing the local search on each of them. We then proceed in rounds of crossover as described above. We note that we assume all crossover operations to be performed in parallel. Hence, the new population is based entirely on the old population and not partially on previously generated individuals of the new generation.

\begin{algorithm2e}
	Let LS$(t)$ denote local search by the \oneonegp with bloat control on tree $t$ for $90 s \log s$ steps, where $s$ is the number of leaves in $t$\;
	\For{$i = 1$ \textbf{to} $\lambda$}{
		Let $t_i$ be a random initial tree of size $\tinit$\;
		$t_i$ $\leftarrow$ LS($t_i$)\;
	}
	\While{optimum not reached}{%
		\For{$i = 1$ \textbf{to} $\lambda$}{
			Choose $m \in \{1,\ldots, \lambda\} \setminus \{i\}$\;
			$t'_i \assign \mbox{join}(t_i,t_m)$\;
			$t''_i$ $\leftarrow$ LS($t_i'$)\;
			\lIf{$f(t''_i) \geq f(t_i)$}{$t_i \assign t''_i$}
		}
	}
	\caption{Concatenation Crossover-GP}
	\label{alg:mimeticCrossover}
\end{algorithm2e}

\subsection{Terminology}
\label{sec:terminology}
In this section we collect standard theorems on stochastic processes that we will use in the proofs. We start with the Chernoff bound.

\begin{theorem}[Chernoff Bound]\label{thm:Chernoff}\cite{dubhashi2009concentration}
Let $b >0$. Let the random variables $X_1, \ldots, X_n$ be independent and take values in $[0,1]$. Let $X= \sum_{i=1}^n X_i$ and $\mu = \Ex{X}$. Then for all $0 \leq \delta \leq 1$,
\[
\Pr[X \leq (1-\delta)\mu] \leq e^{-\delta^2\mu/2}
\]
and
\[
\Pr[X \geq (1+\delta)\mu] \leq e^{-\delta^2\mu/3}.
\]
\end{theorem}

We will also use several drift theorems, which give information about the expected hitting time of a random process if we know the expected progress in each step. The first theorem (``multiplicative drift'') captures the case when the expected change is proportional to the current value of $X_t$.

\begin{theorem}[Multiplicative Drift, tail bound]
	\label{thm:multi_drift_upper_bound}
	\cite{doerr2013adaptive} \newline
	Let $(X_\tau)_{\tau \geq 0}$ be a sequence of random variables taking values in $\{0\} \cup [1, \infty)$. Assume that there is a $\delta >0$ such that
	\[
		\forall \tau \in \N, \, x\in \N_0 \, : \, \Ex{X_\tau \mid X_{\tau-1}=x} \leq (1-\delta) x .
	\]
	Then $\tau_0\coloneqq \min \{ \tau \in \N_0 \mid X_\tau=0\}$ satisfies
	\begin{align*}
	 &\Ex{\tau_0} \leq \frac{1}{\delta} (\ln (X_0) + 1) ; \\
	 &\forall c>0, \, \Prob{\tau_0 > \frac{1}{\delta} (\ln (X_0) + c)} < e^{-c}.
	\end{align*}
\end{theorem}

We also need analogous theorem for constant additive drift and for variable drift.

\begin{theorem}[Additive Drift]
	\label{thm:additive_drift_upper_bound}
	\cite{HeYao:04:drift}
	Let	$(X_\tau)_{\tau\geq 0}$ be a sequence of non-negative random variables over a finite state space $S \subseteq \R$. Let $\tau_0$ be the random variable that denotes the earliest point in time $\tau\geq 0$ such that $X_\tau = 0$. If there exists $c > 0$ such that
	\[
		\Ex{X_\tau - X_{\tau-1} \mid \tau_0 > \tau} \geq c , 
	\]
	then
	\[
		\Ex{\tau_0 \mid X_0} \leq X_0/c .
	\]
\end{theorem}

\begin{theorem}[Variable Drift]\label{thm:variable_drift}
\cite{Joh:th:10}
Let $(X_\tau)_{\tau\geq 0}$ be a sequence of non-negative random variables over a finite state space $S \subseteq [0,\infty)$. Let $x_{\text{min}} \coloneqq \min\{x \in S \mid x >0\}$. Let $\tau_0$ be the random variable that denotes the earliest point in time $\tau\geq 0$ such that $X_\tau = 0$. If there is an increasing function $h: \R^+ \rightarrow \R^+$ such that for all $x \in \mathcal{S}\setminus\{0\}$ and all $\tau\geq 0$,
\begin{align}\label{drifteq:variabledrift1}
\Ex{X_\tau - X_{\tau-1} \mid X_{\tau-1} = x} \geq h(x),
\end{align}
then
\begin{align}\label{drifteq:variabledrift2}
\Ex{\tau_0} \leq  \frac{x_{\text{min}}}{h(x_{\text{min}})} + \int_{x_{\text{min}}}^{X_0} \frac{1}{h(x)} dx.
\end{align}
\end{theorem}

Finally we need some tail bounds on the hitting time for additive drift. The following theorem combines two cases in which we have a drift away from zero. Firstly, starting from zero we expect to need time $n/c$ to hit $n$, and the theorem states that it is exponentially unlikely to need twice as much time. Secondly, starting from $n$ even an exponential number of steps does not suffice to reach zero if the drift pushes in the opposite direction.

\begin{theorem}[Tail Bounds for Additive Drift]
	\label{thm:additive_drift_concentration_upper_bound}
	\cite{DBLP:journals/algorithmica/Kotzing16,oliveto2011negative,oliveto2012erratum}
	Let $(X_\tau)_{\tau\geq 0}$ be a sequence of random variables over $\R$, each with finite expectation. With $\tau_{\geq n} \coloneqq \min \{\tau \geq 0 \mid X_t \geq n \}$ we denote the random variable describing the earliest point at which the random process exceeds $n$, and likewise with $\tau_{\leq 0}$ we denote the earliest point at which the random process drops below zero. Suppose there are $c, K > 0$ such that, for all $t$,
	\begin{enumerate}
		\item $ \begin{aligned}[t]
		&\Ex{X_{\tau-1} - X_\tau \mid X_0, \dots, X_\tau} \geq c, \mbox{ and}
		\end{aligned}$
		\item $ \begin{aligned}[t]
		&|X_\tau - X_{\tau+1}| < K.
		\end{aligned}$	
	\end{enumerate}
Then there is $\rho >0$ such that the following is true for all $n \geq 1$.
\begin{enumerate}[(a)]
\item \[
\Prob{\tau_{\geq n} \geq 2n/c \mid X_0 \geq 0} \leq e^{- n c/(4 K^2)} .
\]
\item 
\[
\Prob{\tau_{\leq 0} \leq e^{\rho n} \mid X_0 \geq n} \leq e^{- \rho n} .
\]
\end{enumerate}
\end{theorem}

For the analysis, it will be helpful to partition the set of leaves into three classes as follows. The set $C^+(t) \subseteq S^+(t)$ of \emph{positive critical leaves} is the set of leaves $u$, whose deletion from the tree results in a decreased fitness. Similarly, the set $C^-(t) \subseteq S^-(t)$ of \emph{negative critical leaves} is the set of leaves $u$, whose deletion from $t$ results in an increased fitness. Finally, the set $R(t) \coloneqq [n] \setminus (C^+(t) \cup C^-(t))$ of \emph{redundant leaves} is the set of all leaves $u$, whose deletion from $t$ does not affect the fitness. Similar as before, we denote $c^-(t) = |C^-(t)|$, $c^+(t) = |C^-(t)|$, and $r(t) = |R(t)|$.

Given a time $\tau \geq 0$, we denote by $t_\tau$ the \gptree after $\tau$ iterations of the algorithm. Additionally, we use $S(\tau), s(\tau), S_i(\tau), \ldots$ in order to denote $S(t_\tau), s(t_\tau), S_i(t_\tau), \ldots$.
Moreover, we apply the standard Landau notation $\BigO(\cdot)$, $\LittleO(\cdot)$, $\Omega(\cdot)$, $\omega(\cdot)$, $\Theta(\cdot)$ as detailed in~\cite{CormenAlgorithms}.

\begin{lemma}\label{lem:random-tree}
Let $\nu >0$ be a constant, and let $t$ be a random \gptree of size $s(t)= \nu n$. Moreover, let $k,\ell \in \N_0$ be constants, and let $N_{k,\ell}$ be the number of variables $i$ with $s_i^+ = k$ and $s_i^- = \ell$. Then with probability $1-e^{-\Omega\left(n^{1/3}\right)}$,
\[
N_{k,\ell} = (1\pm \LittleO(1))\frac{e^{-\nu}(\nu/2)^{k+\ell}}{k! \cdot \ell!} \cdot n = \Theta(n).
\]
\end{lemma}
In particular, for \geqcmajority, \geqfracmajority and \supermajority, a constant fraction of the variables are expressed in $t$, and a constant fraction are not expressed in $t$. 
\begin{proof}
Fix a variable $i$. Each leaf has probability $1/n$ to be an $i$-literal. Therefore, the number of $i$-literals out of $\nu n$ follows approximately a Poisson approximation Po$(\nu)$. More precisely, if $X_1,\ldots,X_n$ are independent Poisson random variables with parameter $\nu$, and $\mathfrak E = \mathfrak E(x_1,\ldots,x_n) \in \{0,1\}$ is any property depending on $n$ integer-valued random variables, then~\cite[Corollary 5.9]{BookMitUp},
\begin{align}\label{eq:poisson_approximation}
\Pr[\mathfrak E(s_1,\ldots,s_n) ] \leq 2e\sqrt{\nu n}\Pr[\mathfrak E(X_1,\ldots,X_n)].
\end{align}
In particular, by definition of the Poisson distribution, $\Pr[X_i = k +\ell] = e^{-\nu}\nu^{k+\ell}/(k+\ell)! \eqqcolon q$ independently for every $i$. Therefore, if we set $Y_i \coloneqq 1$ if $X_i= k+\ell$, and $Y_i\coloneqq 0$ otherwise, then $Y\coloneqq \sum_{i=1}^n Y_i$ is the number of random variables with value $k +\ell$. We have $\EE[Y] = q n$, and by the Chernoff bound~\ref{thm:Chernoff} with $\delta = n^{-1/3}$,
\[
\Pr\left[Y \not\in (1\pm \delta)q n\right]\leq 2e^{-q n^{1/3}/3}.
\]
Thus, if we call $L$ the set of $i\in [n]$ such that there are exactly $k+\ell$ $i$-literals, then by~\eqref{eq:poisson_approximation}
\[
\Pr\left[|L| \not\in (1\pm \delta)q n\right]\leq 2e\sqrt{\nu n}e^{-q n^{1/3}/3} = e^{-\Omega\left(n^{1/3}\right)}.
\]
Therefore, we may assume that $|L| \in (1\pm \delta)q_{k,\ell} n$. Let $L' \subseteq L$ be the set of variables $i$ with $s_i^+=k$ and $s_i^-=\ell$. Then each $i\in L$ has probability $q' \coloneqq \binom{k+\ell}{k}/2^{k+\ell} = \Theta(1)$ to be in $L'$, independently for each $i$. Hence, with $\mu \coloneqq \EE\left[|L'| \,\middle|\, |L| \in (1\pm \delta)q n\right] \geq (1- \delta)q'q n$, a similar application of the Chernoff bound as before shows
\[
\Pr\left[|L'| \not\in (1\pm \delta)^2q'q n \,\middle|\, |L| \in (1\pm \delta)q n\right]\leq e^{-\delta^2(1-\delta)q'qn/3}.
\]
Since $q'q = e^{-\nu}(\nu/2)^{k+\ell}/(k! \ell!)$, the lemma follows.

The remark on expressed variables follows by setting $k=c$ and $\ell=0$ for \geqcmajority and $k=1$ and $\ell=0$ for the other two functions. The remark on unexpressed variables follows by setting $k=0$ and $\ell=0$.
\end{proof}

\section{(1+1) GP with Bloat Control}
\label{sec:bloatcontrol}

In this section we study how local search with bloat control performs on the given fitness functions. Theorem~\ref{thm:+cMajority_lower-bound_constant-stepsize} shows that for small initial trees \geqcmajority cannot be efficiently optimized, while Theorem~\ref{thm: 2-3Majority_bc_constant-stepsize} shows that this is possible for \geqfracmajority. Finally, Theorem~\ref{thm: 2-3SuperMajority_bc_constant-stepsize} considers \supermajority.

\begin{theorem}\label{thm:+cMajority_lower-bound_constant-stepsize}
	Consider the \oneonegp on \geqcmajority with bloat control on the initial tree with size $\tinit < n$. If $c >1$, with probability equal to $1$, the algorithm will never reach the optimum.
\end{theorem}

\begin{proof} 
	We assume optimistically that there are no negative leaves in the initial tree, since they slow down the optimization.
	Let, for all time steps $\tau$, $X_{\tau}$ be the set of variables $i \in [n]$ with no $\var_i$ in the tree after step $\tau$. In order to reach the optimum the \oneonegp has to reach a \gptree after $\tau$ steps with $|X_{\tau}| = 0$. Since $\tinit <n$, there is at least one $j \in [n]$ without $\var_j$ in the initial tree, hence $|X_{0}| \geq 1$. We are going to show that $|X_\tau| \geq 1$ for all $\tau$.
	
	Let $i_0 \in X_0$.
	Since the bloat control will reject insertions of redundant leaves (in particular of $\var_{i_0}$), the only way to express ${i_0}$ is to substitute at least one leaf $\var_j$ with $\var_{i_0}$. Furthermore, this substitution cannot reduce the fitness and hence $\var_j$ has to be a redundant leaf. If no redundant leaf exists ${i_0}$ cannot become expressed. Otherwise, $|X_1|\geq 1$ because either there is only one $\var_j$ and by the substitution $j$ is in $X_1$ or there are more than one $\var_j$ and (by counting) $|X_0|\geq 2$. We observe that deletions and insertions can only reduce the number of redundant leaves. Hence, iterating the argument with $i_1 \in X_1$ gives the desired result.
\end{proof}

\begin{theorem}\label{thm:2-3-Majority_lower-bound_constant-stepsize}
	Consider the \oneonegp on \geqfracmajority with bloat control on the initial tree with size $\tinit < n$. The expected time until the algorithm computes the optimum is in $\Omega(n \log n)$.
\end{theorem}
\begin{proof}
	Let $t^*$ be the optimal solution, which consists of a leaf $\var_i$ for each variable $i \in [n]$. The set of variables $[n]$ decomposes into four sets:
	\begin{description}
		\item[$A$:] $i \in [n]$ without any leaf $\var_i$ or $\nonvar_i$ in $\tinit$,
		\item[$B$:] $i \in [n]$, such that $\tinit$ contains a leaf $\nonvar_i$ but no $\var_i$,
		\item[$C$:] $i \in [n]$, such that $\tinit$ contains a leaf $\nonvar_i$ and a $\var_i$,
		\item[$D$:] $i \in [n]$, such that $\tinit$ contains a leaf $\var_i$ but no $\nonvar_i$.
	\end{description}
	Since $\tinit < n$ the initial tree cannot be the optimal one, furthermore there has to exist a $j \in A$. In order to reach the optimum $t^*$, the algorithm needs to add a leaf $\var_i$ for every $i \in A$. For every $i \in B$ the algorithm needs to insert $\var_i$ as well, but due to the bloat control it will reject such a move if it does not increase the fitness. Hence, the algorithm needs to delete every $\nonvar_i$ prior to inserting $\var_i$, where the deletion of the last $\nonvar_i$ can alternatively be done by substituting it with $\var_i$. We observe that, for every $i \in A \cup B$ the algorithm needs to insert $\var_i$ or substitute a redundant leaf with $\var_i$. Since doing so is essentially equivalent to a \emph{Coupon Collector}, we obtain a run time of $\Omega (n \log k)$, where $k$ is the cardinality of $A \cup B$.
	It remains to show, that $n \log k \in \Omega (n \log n)$.
	
	We show $|C \cup D| \leq 2n/3$. For this we note that each entry of the initial tree is a positive leaf with probability $1/2$. Therefore, in expectation, the initial tree contains $\tinit/2$ positive leaves. Furthermore, the probability that the amount of positive leaves in the initial tree is higher than the expectation falls exponentially fast. We can observe this behavior due to a Chernoff bound. For all $j \in [\tinit]$, let the indicator variable $X_j$ be $1$ if the leaf at position $j$ in the initial tree is a positive one. Else, the indicator variable is $0$. We obtain
	\[
		\Prob{\sum_{j=1}^{\tinit} X_j \geq \left(1+\frac{1}{3} \right) \frac{\tinit}{2} } \leq e^{- \tinit/36}
	\]
	
	Hence, with high probability the initial tree will contain less than $2 \, \tinit/3$ positive leaves, yielding $|C \cup D| < 2n/3$ and the claim follows.
\end{proof}

Next we state the upper bound for the performance on \geqfracmajority. The proof of Theorem~\ref{thm: 2-3Majority_bc_constant-stepsize} is almost identical to the one of Theorem~4.1 in \cite{doerr2017bounding}, the bounds stated in Lemma~4.2 and Lemma~4.1 in \cite{doerr2017bounding} need to be suitably adjusted, since these do not hold for \geqfracmajority.

\begin{theorem}\label{thm: 2-3Majority_bc_constant-stepsize}
	Consider the \oneonegp on \geqfracmajority with bloat control on the initial tree with size $\tinit$. The expected time until the algorithm computes the optimum is in $\BigO(n \log n + \tinit )$.
\end{theorem}
\begin{proof}
	
	Let $t$ be a \gptree over $n$ literals and denote the number of expressed literals of $t$ by $v(t)$. For a best-so-far \gptree of the \oneonegp we denote the size of the initial \gptree by $\tinit$. Both parameters $n$ and $\tinit$ are considered to be given. We partition the set of leaves (again) by observing the behavior when deleting the leaf. This introduces the set of redundant leaves $R(t)$, of critical positive leaves $C^+(t)$ and of critical negative leaves $C^-(t)$ with their respective cardinality denoted by using lower case letters. We obtain
	\[
		s(t) = r(t) + c^+(t) + c^-(t).
	\]
	First, we observe that an adjustment of~\cite[Lemmas 4.2, 4.3]{doerr2017bounding} yields the following:
	\begin{align*}
		c^-(t) &\leq r(t) \quad \mbox{and} \quad
		c^+(t) \leq 2 r(t) + v(t).		
	\end{align*}
	For a best-so-far \gptree $t$ let $t'$ be the \gptree after one additional round of mutation and selection in the \oneonegp. By bounding the drift with respect to a suitable potential function $g$, i.e. the expected change $g(t) - g(t')$ denoted by $\Delta (t)$, we are going to obtain the bound for the optimization time due to the Variable Drift Theorem~\ref{thm:variable_drift}.
	
	We associate with $t$ the potential $g(t)$ given by
	\[
		g(t) = n-v(t) + s(t) - v(t) = n+s(t)-2v(t).
	\]
	This potential is $0$ if and only if $t$ contains no redundant leaves and for each $i \leq n$ there is exactly one $x_i$. We observe that the drift cannot be negative since the algorithm only does $1$ mutation in each iteration and the bloat control will reject insertions of new redundant leaves.
	
	{\bf Case $1$:} assume $r(t) \geq v(t)$. We obtain
	\[
		s(t) = r(t) + c^+(t) + c^-(t) \leq 5 r(t).
	\]
	Let $\event{E}_1$ be the event, that the algorithm deletes a redundant leaf. The drift in this case will be $1$ and the probability for such a move is $1/3$ for a deletion followed by at least $1/5$ to choose a redundant leaf. We obtain
	\[
		\Ex{\Delta(t)} \geq \Ex{\Delta(t) \mid \event{E}_1} \, \Prob{\event{E}_1} \geq \frac{1}{15}.
	\]
	{\bf Case $2$:} Suppose $r(t) < v(t)$ and $s(t) \leq n/2$. In particular, we have for at least $n/2$ many $i \in [n]$ that there is neither $\var_i$ nor $\nonvar_i$ present in $t$. Let $\event{E}_2$ be the event that the algorithm inserts such a $\var_i$. The probability to choose such an $\var_i$ is at least $1/4$ and the probability that the algorithm chooses an insertion is $1/3$. We obtain
	\[
		\Ex{\Delta(t)} \geq \Ex{\Delta(t) \mid \event{E}_2} \Prob{\event{E}_2} \geq \frac{1}{12}.
	\]
	{\bf Case $3$:} assume $r(t) < v(t)$ and $s(t) > n/2$. In particular, the tree can contain at most $5n$ leaves due to
	\begin{align}\label{eq:case_3}
		s(t) &= r(t) + c^+(t) + c^-(t) \leq r(t) + 2 r(t) + v(t) + r(t) \leq 5 v(t) \leq 5n.
	\end{align}
	Hence, the probability that an operation chooses a specific leaf $v$ is
	\[
		\frac{1}{5n} \leq \Prob{\mbox{choose leaf }v} \leq \frac{2}{n}.
	\]
	Let $A$ be the set of $i$ without $\var_i$ or $\nonvar_i$ in $t$ and let $B$ be the set of $i$ with exactly one $\var_i$ but no $\nonvar_i$ in $t$. Recall that $R(t)$ is the set of redundant leaves of $t$. For every $j$ in $A$ let $\event{A}_j$ be the event that the algorithm adds $\var_j$ somewhere in $t$. For every $j$ in $R(t)$, let $\event{R}_j(t)$ be the events that the algorithm deletes $j$. Finally, let $\event{A}'$, $\event{R}'$ be the event, that one of the $\event{A}_j$, respectively $\event{R}_j(t)$, holds.
	We obtain
	\begin{align*}
		\Ex{\Delta (t) \mid \event{A}_j} &= 1 \quad \mbox{and} \quad
		\Ex{\Delta (t) \mid \event{R}_j(t)} = 1,
	\end{align*}
	as well as 
	\begin{align*}
		\Prob{\event{A}_j} &\geq \frac{1}{6n} \quad \mbox{and} \quad
		\Prob{\event{R}_j(t)} \geq \frac{1}{15n}.
	\end{align*}
	We observe
	\[
		|A|+|R(t)| \geq r(t).
	\]
	Furthermore, we noticed that for any literal $j$, which is not in $B$ or $A$, there has to exists at least one redundant leaf $\var_i$ or $\nonvar_i$. We obtain $|A|+|B| + |R(t)| \geq n$ and thus
	\[
		|A| + |R(t)| \geq n - v(t).
	\]
	Additionally, by (\ref{eq:case_3}),
	\[
		s(t) - v(t) \leq 4r(t) \leq 4 (|A| + |R(t)|),
	\]
	which in conjunction with the above inequality yields
	\[
		5(|A| + |R(t)|) \geq n - v(t) + s(t) - v(t) = g(t).
	\]
	We obtain the expected drift
	\begin{align*}
		\Ex{\Delta(t)} &\geq \Ex{\Delta(t) \mid (\event{A}' \vee \event{R}')} \Prob{\event{A}' \vee \event{R}'} \\
		&= \sum_{j \in A} \Ex{\Delta(t) \mid \event{A}_j} \Prob{\event{A}_j}
		+ \sum_{j \in R(t)} \Ex{\Delta(t) \mid \event{R}_j(t)} \Prob{\event{R}_j(t)} \\
		&\geq |A| \frac{1}{6n} + |R(t)|\frac{1}{15n} \geq \frac{1}{15n} (|A| + |R(t)|) 	\geq \frac{g(t)}{75 n} .
	\end{align*}
	We distinguish two cases. For $g(t) \leq 5n$, we obtain a multiplicative drift of at least $g(t)/(75n)$, while for $g(t) > 5n$ the drift is at least $1/15$. We now apply the Variable Drift Theorem~\ref{thm:variable_drift} with $h(x) = \min \{ 1/15, x/(75n) \}$, $X_0 = \tinit + n$ and $x_{\min} =1$, which yields the desired bound on the time $\tau$ until $g_{\tau} =0$.
	\begin{align*}
		\Ex{\tau \mid g(t) = 0} &\leq \frac{1}{h(1)} + \int_{1}^{\tinit + n} \frac{1}{h(x)} ~dx = 75n + 75n \int_{1}^{5n} \frac{1}{x} ~dx + 15 \int_{5n+1}^{\tinit + n} 1 ~dx \\
		&= 75n (1 + \log (5n)) + 15 (\tinit -4n -1) \leq 75n \log (5n) + 15 \tinit + 15n. 
	\end{align*}
	This establishes the theorem.
\end{proof}
\begin{corollary}
	Consider the \oneonegp on \geqfracmajority with bloat control on the initial tree with size $\tinit < n$. The expected time until the algorithm computes the optimum is in $\BigO(n \log n)$.
\end{corollary}

We turn to \supermajority with Theorem~\ref{thm: 2-3SuperMajority_bc_constant-stepsize}. The proof is based on the following lemma showing that redundant leaves will be removed with sufficient probability. Hence, insertions of positive literals can increase fitness.

\begin{lemma}\label{lem:bloatcontrol-supermajority}
	Consider the \oneonegp on \supermajority with bloat control with $n\geq 55$ on the initial tree with size $\tinit < n$. With probability at least $1-(\tau/(n\log^2 n))^{-1/(1+4/\sqrt{\log n})}$ the algorithm will delete any given negative leaf of the initial tree within $\tau \geq n\log^2 n$ rounds. For a positive redundant leaf, with the same probability it will either be deleted or turned into a positive critical leaf.
\end{lemma}
\begin{proof}
Consider any set of $b$ iterations. The expected number of (attempted) insert iterations is $b/3$. Thus, we can use the Chernoff bounds (Theorem~\ref{thm:Chernoff}) to bound the size of the tree assuming pessimistically that all insertions are accepted and no deletions are accepted. Thus, for $x \leq b/3$, the growth of the tree over these $z$ iterations is at most $b/3 + x$ with probability at most
$
\exp\left(-x^2/b\right).
$

We partition the iterations of the algorithm into consecutive  blocks of length $b = \log(n)^2$. Let $x = 2\log(n)^{1.5}$. Let $\event{A}_i$ be the event that the size grew, during the $i$th block, by at most $b/3+x$. For $n \geq 55$ we have $x \leq b$, so for all $i$, $\Prob{\overline{\event{A}_i}} \leq n^{-4}$.

We will henceforth condition on the event $\bigcap_{i=1}^{n^3} \event{A}_i$, which has a probability of at least $1-1/n$. Thus, after $i < n^3$ blocks, the tree grew by at most $i(b/3+x)$. Assume that the designated leaf is either negative or that it does not turn into a non-redundant leaf. By $\event{B}_j$ we denote the event that the leaf is \emph{not} deleted in iteration $j$. We have, for each iteration $j$ within block $i$, $\Prob{\event{B}_j} \leq 1 - 1/(3(\tinit + i(b/3 + x)))$. Hence, the probability that the designated leaf is \emph{not} deleted in block $i$ is 
\begin{align*}
\Prob{\bigcap_{j=bi}^{b(i+1)-1} \event{B}_j} & \leq \prod_{j=bi}^{b(i+1)-1} (1 - 1/(3(\tinit + i(b/3 + x))))\\
 & = (1 - 1/(3(\tinit + i(b/3 + x))))^{b} \leq \exp(-b/(3\tinit + ib+3ix)).
\end{align*}

We want to compare the denominator with $ib$. Thus we write $3ix = 3ib/\sqrt{\log n}$. Moreover, for $i \geq n$ we have $3\tinit \leq 3n \leq ib/\sqrt{\log n}$. Hence, the probability of \emph{not} deleting the designated leaf with the first $\tau/b < n^3/b$ blocks is at most
\begin{align*}
 & \prod_{i=1}^{\tau/b} \exp(-b/(3\tinit + ib+3ix)) \leq \prod_{i=n}^{\tau/b} \exp(-b/(ib+4ib/\sqrt{\log n}))\\
  & =  \exp\left(-(1/(1+4/\sqrt{\log n})) \sum_{i=n}^{\tau/b} 1/i\right) \leq  \exp\left(-\ln(\tau/(bn))/(1+4/\sqrt{\log n})\right)\\
  & = (\tau/(bn))^{-1/(1+4/\sqrt{\log n})}.
\end{align*}
\end{proof}

\begin{theorem}\label{thm: 2-3SuperMajority_bc_constant-stepsize}
Consider the \oneonegp on \supermajority with bloat control on an initial tree with size $\sinit < n$, and let $\eps >0$. Then, the algorithm will express every literal after $n^{2+\eps}$ iterations with probability $1-\LittleO(1)$.
\end{theorem}
\begin{proof}
Consider the set $A$ of leaves which are redundant or negative in the initial tree. Due to bloat control, the number of such leaves can only decrease over time. We call a leaf $i \in A$ \emph{bad}, if it is not deleted or turned into a positive critical leaf in the first half of the iterations. By Lemma~\ref{lem:bloatcontrol-supermajority}, each fixed leaf in $A$ has probability at most $1-n^{-(1+\eps)/(1+4\sqrt{\log n})}/(2\log^2n)$. If $n$ is sufficiently large, then $(1+\eps)/(1+4\sqrt{\log n}) \geq 1+\eps/2$. Therefore, the probability that the leaf is bad is at most $n^{-1-\eps/2}/(2\log^2 n) = \LittleO(1/n)$. By a union bound, the probability that there is \emph{any} bad leaf is $\LittleO(1)$. In particular, with probability $1-\LittleO(1)$, after the first half of the iterations there are no negative leaves left. 

Without negative literals, in the second half of the algorithm, an unexpressed variable has probability $1/(6n)$ to become expressed in each step. Therefore, the probability that it is not expressed after the second half is at most $(1-1/(6n))^{n^{2+\eps}/2} = \LittleO(1/n)$. Thus the expected number of unexpressed variables is $\LittleO(1)$. By Markov's inequality, with probability $1-\LittleO(1)$ all variables are expressed after the second half of the algorithm, as claimed.
\end{proof}

\section{(1+1) GP without Bloat Control}
\label{sec:bloat}
In this section we study the fitness function \supermajority, which facilitates bloat of the string.

\begin{theorem}\label{thm:geqfracMajority_lower-bound_constant-stepsize}
	For any constant $\nu >0$, consider the \oneonegp without bloat control on \supermajority on the initial tree with size $\tinit = \nu n$. There is $\eps = \eps(\nu)>0$ such that, with probability $1-\LittleO(1)$, an $\eps-$fraction of the indices will never be expressed. In particular, the algorithm will never reach a fitness larger than $(2-2\eps)n$.
\end{theorem}

We commence with some preparatory lemmas before proving the theorem. First, we analyze how the size of the \gptree evolves over time. We recall that $s(\tau)$ is the number of leaves of the \gptree at time~$\tau$. 
\begin{lemma}\label{lem:length}
	There is a constant $0<\eta \leq 1$ such that, with probability $1-\LittleO(1)$, for all $\tau \geq 0$ we have $s(\tau) \geq \eta \tau$.
\end{lemma}
\begin{proof}
Let $v(\tau)$ be the number of expressed literals at time $\tau$. By Lemma~\ref{lem:random-tree}, with high probability there are at least $\eta' n$ expressed variables in the initial tree, for some constant $\eta'>0$. So $v(0) \geq \eta'n$. 

Now we examine how the number $c^+(\tau)$ of positive critical leaves evolves over time. We claim that with high probability, for all $\tau\geq 0$,
\begin{align}\label{eq:bound-on-cplus}
c^+(\tau) \geq \eta' \tau/12.
\end{align} 
Note that a mutation that decreases the number of expressed literals also decreases the fitness and is rejected. Thus $v(\tau)$ is increasing in $\tau$, and hence $c^+(\tau) \geq v(\tau) \geq v(0) \geq \eta' n$ for all $\tau\geq 0$. This already implies \eqref{eq:bound-on-cplus} for $\tau \leq 12n$. Similarly, an offspring in which a critical positive leaf is deleted is never accepted, and an offspring in which a critical positive leaf is substituted can only be accepted if the leaf is substituted by another positive critical leaf. Therefore, $c^+(\tau)$ is also increasing in $\tau$. 
Moreover, in each step of the algorithm, with probability $v(\tau)/(6n) \geq \eta'/6$ a new positive literal is created that is already expressed. In this case, the number of positive critical literals increases by one. Hence, for all $\tau\geq 0$,
\begin{align}\label{eq:drift-of-cplus}
\Ex{c^{+}(\tau+1)-c^{+}(\tau)} \geq \eta'/6.
\end{align}
For any fixed $\tau\geq 12n$, by the tail bounds on the Additive Drift Theorem~\ref{thm:additive_drift_concentration_upper_bound}, $\Pr[c^{+}(\tau) \leq \eta' \tau/12] \leq e^{-\rho \tau}$, where $\rho>0$ is the constant from Theorem~\ref{thm:additive_drift_concentration_upper_bound}. 
Therefore, by a union bound,
\[
\Pr[\exists \tau \geq 12 n \mid c^{+}(\tau) \leq \eta' \tau/12] \leq \sum_{\tau=12n}^{\infty}e^{-\rho \tau} = \LittleO(1).
\]
This proves \eqref{eq:bound-on-cplus} for all $\tau\geq 0$. The lemma now follows simply from $s(\tau) \geq c^+(\tau)$ with $\eta \coloneqq \eta'/12$.
\end{proof}

In order to continue we need some more terminology. For an index $i \in [n]$, we recall that $s^+_i(\tau)$ and $s^-_i(\tau)$ denote the number of $\var_i$- and $\nonvar_i$-literals at time $\tau$, respectively, and $s_i(\tau) \coloneqq s^+_i(\tau) + s^-_i(\tau)$. We call index $i$ \emph{touched} in round $\tau$, if a literal $\var_i$ or $\nonvar_i$ is deleted, inserted or substituted, or if a literal is substituted by $\var_i$ or $\nonvar_i$. We call the touch \emph{increasing} if it is either an insertion or if a literal is substituted by $\var_i$ or $\nonvar_i$. We call the touch \emph{decreasing} if it is a deletion or substitution of a $\var_i$ or $\nonvar_i$ literal. We note that in exceptional cases a substitution may be both increasing and decreasing. Let $\rho_i(\tau)$ be the number of increasing touches of $i$ up to time $\tau$. 
We call a decreasing step \emph{critical} if it happens at time $\tau$ with $s_i(\tau) \leq \eta \tau/(4n)$, and we call $\gamma_i(\tau)$ the number of critical steps up to time $\tau$. Finally, we call a round \emph{accepting} if the offspring is accepted in this round.

The approach for the remainder of the proof is as follows. First, we will show that in the regime, where critical steps may happen (i.e, $s_i(\tau) \leq \eta \tau/(4n)$), it is more likely to observe increasing than decreasing steps. The reason is that a step is only critical if there are relatively few $i$-literals, in which case it is unlikely to delete or substitute one of them, whereas the probability to insert an $i$-literal is not affected. It will follow that $s_i(\tau)$ grows with $\tau$, since otherwise we would need many critical steps. Finally, if $s_i(\tau)$ keeps growing it becomes increasingly unlikely to obtain a $2/3$ majority. In order to state the first points more precisely we fix a $j_0 \in \N$ and call an index $i$ \emph{bad} (or more precisely, $j_0$-bad) if the following conditions hold: for all $\tau \geq j_0n$ and $\tau_0 \coloneqq j_0 n$
\begin{center} $
\begin{array}{rlrl}
	\mbox{(A)}	& s^{+}_i(\tau_0) \leq s^{-}_i(\tau_0) \leq j_0 \quad & 
	\mbox{(B)}	& \tau/(2n) \leq \rho_i(\tau) \leq 2\tau/n \\ 
	\mbox{(C)}	& \gamma_i(\tau) \leq 2\tau/n & 
	\mbox{(D)}	& 	s_i(\tau) \geq \eta \tau/(8n).
\end{array} 
$
\end{center}
In particular, in $(A)$ $\var_i$ is not expressed at time $\tau_0$.

\begin{lemma}\label{lem:bad-indices}
	For every fixed $i_0 > 0$, with probability $1-\LittleO(1)$ there are $\Omega(n)$ bad indices.
\end{lemma}
\begin{proof}
We will show that a given index $i$ has probability $\Omega(1)$ to be bad. It is more technical to show concentration, and we only give a sketch of the argument at the end of the proof. So fix an index $i$. We will first show individually\footnote{with some slight complications for (D)} that (B), (C), (D) all hold with rather large probability, say, with probability $0.9$. Then by a union bound, the probability that one of them does \emph{not} hold is at most $0.3 <1$. Finally we show that (A) holds with probability $\Omega(1)$, and that (B), (C), (D) still hold with sufficiently large probability if we condition on (A). 

We start with (B). Note that the probability for a round to be increasing is always $2/(3n)$, independent of the current \gptree. In particular, it is independent of whether $\var_i$ is expressed or not. Therefore, the expected number of increasing rounds up to time $\tau$ is $2\tau/(3n) = 6\tau/(12n)$. Assume that for some $\tau_j$ of the form $\tau_j = (6/5)^j j_0 n$ the inequality $\rho_i(\tau_j) \geq 5\tau_j/(12n)$ holds, which is stronger than the first inequality in (B). Then for all $\tau' \in [\tau_j,\tau_{j+1}]$ it will follow that $\rho_i(\tau) \geq \rho_i(\tau_j) \geq 5\tau_j/(12n) \geq \tau/(2n)$, as in (B). Therefore, it suffices to show that $\rho_i(\tau_j) \geq 5\tau_j/(12n)$ holds for all $j \in \N$ to conclude the first inequality in (B). However, by the Chernoff bound the probability that this is violated for any large $j$ is
\[
\Pr[\exists j \geq j_0 \mid \rho_i(\tau_j) \leq 5\tau_j/(12n)] \leq \sum_{j \geq j_0} e^{-\tau_j/50} = e^{-\Omega(i_0)},
\]
which is small if $j_0$ is sufficiently large. 
An analogous argument shows the second inequality $\rho_i(\tau) \leq 2\tau/n$ of (B). Thus, every index $i$ has large probability (at least $0.9$) to satisfy (B). 

For (C), essentially the same argument applies again. By definition a critical round can only occur if $s_i(\tau) \leq \eta \tau/(4 n)$. By Lemma~\ref{lem:length} this implies $s_i(\tau) \leq s(\tau)/(4n)$, so the probability to choose a deletion or substitution that hits an $i$-literal is at most $2/3 \cdot 1/(4n) = 1/(6n)$. The rest follows as for (B), with room to spare.

For (D), assume that for some $\tau \geq j_0 n$ we have $s_i(\tau) \leq \eta \tau/(4 n)$. Then as for (C), the probability that a step is decreasing is at most $1/(6n)$, while the probability of an increasing step is $2/(3n)$. Therefore, if we consider the random variable $X(\tau) \coloneqq s_i(\tau) - \eta \tau/(4n)$ then $X(\tau)$ has a \emph{positive} drift whenever $X(\tau) \leq 0$,
\begin{align}\label{eq:drift-of-elli}
\Ex{X(\tau+1)-X(\tau) \mid X(\tau) \leq 0} \geq \frac{1}{2n} - \frac{\eta}{4n} \geq \frac{1}{4n}.
\end{align}
We claim that this renders it unlikely that $s_i(\tau) < \eta \tau/(8n)$ for some $\tau\geq j_0 n$. Indeed, assume that such a $\tau = \tau_0$ exists, and let $\tau_0' \coloneqq \max\{\tau \in [j_0n,\tau_0] \mid X(\tau) \geq 0\}$ be the last point in time at which $X(\tau)$ was non-negative. (Assume for the moment that $X(j_0 n) \geq 0$ so that such a time exists.) For technical reasons that will become clear later, let $\tau_0''$ be the time of the first decreasing step \emph{after} $\tau_0'$. Then from $\tau_0''$ to $\tau_0$, $X(\tau)$ performs a random walk with positive drift, which declines from $X(\tau_0'') \geq -1$ to $X(\tau_0) < - \eta \tau_0/(8n) < -\eta \tau_0''/(8n)$ without hitting $X(\tau)\geq 0$ in the meantime. However, for any fixed $\tau_0''$, this happens with probability at most $e^{- \Omega(\tau_0''/n)}$ by Theorem~\ref{thm:additive_drift_concentration_upper_bound}. Moreover, assume for a moment that (C) holds. Then since $\tau_0''$ is a critical step, there are only a limited number of candidates for $\tau_0''$, because the $j$-th critical step does not happen before $\tau_j = jn/2$. Hence, the probability to have a random walk that declines from $X(\tau_0'') \geq -1$ to $X(\tau_0) < -\eta \tau_0''/(8n)$ for some $\tau_0'' \geq j_0 n$ is at most $\sum_{j\in \N, \tau_j \geq j_0 n} e^{-\Omega(\tau_j/n)} = e^{-\Omega(j_0)}$. Therefore, the only possibilities that (D) fails with some $\tau_0'' \geq j_0 n$ is that either (C) fails (which is unlikely), or that there is a strongly declining random walk (which is also unlikely). This shows that (B), (C), (D) all happen with probability at least $0.9$ if we assume that $X(j_0 n) \geq 0$.

It remains to argue that it is sufficiently likely that both $X(j_0 n) \geq 0$ and (A) holds. Consider the event that there are exactly $j_0$ increasing and no decreasing steps until time $j_0 n$, and that all the increasing steps until time $j_0n$ introduce negative literals. This event has probability $\Omega(1)$, and it implies $X(j_0 n) \geq 0$ and (A). Moreover, each of (B), (C), (D) still holds with probability at least $0.9$ if we condition on this event. Therefore, (A), (B), (C), (D) all hold simultaneously with probability $\Omega(1)$.

Altogether, we have shown that a literal $i$ has probability $\Omega(1)$ to be bad. Therefore, the expected number of bad literals is $\Omega(n)$. It remains to show concentration, for which we only give a sketch. We would like to use a concentration bound like the Chernoff bound, but unfortunately for two indices $i$ and $i'$ the events ``$i$ is bad'' and ``$i'$ is bad'' are not independent. For example, if we add a $\var_i$ literal at time $\tau$, then we cannot add a $\var_{i'}$ literal in the same iteration. However, we can couple the process to an uncovering process with independent steps as follows. Consider a random set $A \subseteq [n]$ of size $\eps n$, for some small constant $\eps >0$. Then by the Chernoff bound, the literals corresponding to $A$ will constitute at most a $2\eps$ fraction of the leaves in the \gptree, and they will only affect a $2\eps$ fraction of all the iterations. Here we use implicitly that there are $\Omega(n)$ expressed literals from the beginning, and thus there is no single literal which constitutes a large fraction of the leaves. Then for the indices in $A$, we reveal one by one whether they are bad or not. The crucial advantage is that even after uncovering for the first indices in $A$ whether they are bad, the remaining indices still have probability $\Omega(1)$ to be bad. The reason is that our proof that the probability is $\Omega(1)$ is still valid if we have information about a $2\eps$ fraction of the rounds. Thus we may couple the process of uncovering to a process where we flip independent coins for each $i\in A$, and the Chernoff bound tells us that the number of bad indices in $A$ is concentrated. We omit the details.
\end{proof}

\begin{lemma}\label{lem:bad-indices-not-expressed}
	Every bad index has probability $\Omega(1)$ that it is never expressed, independent of the other bad indices.
\end{lemma}
We note that Lemmas~\ref{lem:bad-indices} and~\ref{lem:bad-indices-not-expressed} imply Theorem~\ref{thm:geqfracMajority_lower-bound_constant-stepsize} by a straightforward application of the Chernoff bound.
	\begin{proof}[Proof of Lemma~\ref{lem:bad-indices-not-expressed}]
		Let $i$ be a bad index. For $j\geq 0$, let $\tau_j$ be the $j$-th accepting round after $j_0 n$ in which $i$ is touched. (We note that the $i$-literals have no effect on the fitness while $i$ is not expressed. However, an offspring may be rejected in substitutions.) We will study how $\delta(j) \coloneqq s_i^+(\tau_j)-s_i^-(\tau_j)$ evolves over time. 
		
		We first show that if $\delta(j)$ is not too large ($\delta(j) < \eta j /144$) then $i$ is not expressed. Initially (at time $j_0n$), by (A) there are at most $2j_0$ $i$-literals. Since the number of $i$-literals is non-negative, the number of decreasing accepting rounds exceeds the number of increasing accepting rounds by at most $2j_0$. Additionally, by (B) the number of increasing accepting rounds before time $\tau \geq j_0n$ is at most $2t/n$. Therefore, the total number of accepting rounds that touch $i$ before time $\tau \geq j_0n$ is at most $4\tau/n +2j_0 \leq 6t/n$. In particular, this implies $\tau_j \geq jn/6$. Due to (D) we have $s_i(\tau_j) \geq \eta \tau_j/(8n) \geq \eta j /48$. Thus, if $s_i^+(\tau_j) \geq \tfrac23 s_i(\tau_j)$ this implies $\delta(j) \geq \tfrac13 s_i(\tau) \geq \eta j / 144$. Conversely, if $\delta(j) < \eta j /144$ then $i$ is not expressed.
		
		We proceed by studying how $\delta(j)$ evolves over time. In order to avoid border cases we will show that, with probability $\Omega(1)$, we have $\delta(j) \leq \eta j /144 -1$ for all $j\in \N$. Moreover, we will treat a substitution that changes $\delta(j)$ by $2$ as two consecutive operations. We note that in this regime $i$ cannot become expressed by a single step. Thus, the selection operator does not discriminate between $\var_i$ and $\nonvar_i$. Regardless of $\delta(j)$, increasing operations have the same probability to introduce $\var_i$ and $\nonvar_i$. Regarding decreasing operations, if $\delta >0$ then it is more likely to select a $\var_i$-literal than a $\nonvar_i$-literal (because there are more $\var_i$-literals than $\nonvar_i$-literals). Hence, it is more likely to decrease $\delta$ than to increase it. Likewise, for $\delta <0$ it is more likely to increase $\delta$ than to decrease it. Therefore, $\delta(j)$ performs a random walk with $\delta(j+1) = \delta_j \pm 1$ and 
		\begin{align*}
		\Pr[\delta(j+1) = \delta(j)+1] &\geq 1/2,  \text{ if $\delta <0$};\\
		\Pr[\delta(j+1) = \delta(j)+1] &\leq 1/2,  \text{ if $\delta >0$}.
		\end{align*}
		Therefore, for any $k \geq 0$ the probability that $|\delta(j)| > k$ is at most the probability that an unbiased random walk takes a value $>k$ after $j$ steps. This latter probability is $2\Pr[\text{Bin}(j,1/2) \geq j/2 + k/2]$, where Bin is the binomial distribution. In particular, for $k=\eta j/144 -1$
		\begin{align*}
		\Pr[|\delta(j)| \geq k] & \leq 2\Pr[\text{Bin}(j,1/2) \geq j/2 + k/2]  = e^{-\Omega(k)} = e^{-\Omega(j)},
		\end{align*}
		where the last step follows from the Chernoff bound. Due to a union bound over all $j \geq j_1$ the probability that there is $j \geq j_1$ with $|\delta(j)| \geq \eta j/144 -1$ is $e^{-\Omega(j_1)}$. Therefore, it becomes more and more unlikely that the literals ever becomes expressed. It remains to choose (somewhat arbitrarily) a sufficiently large constant $j_1$ and to observe that, with probability $\Omega(1)$, we have $\delta <0$ in the first $j_1$ rounds. This concludes the proof.
	\end{proof}

\section{Concatenation Crossover GP}
\label{sec:crossover}
In the following we will study the performance of the Concatenation Crossover GP (Algorithm~\ref{alg:mimeticCrossover}) on \geqcmajority and \geqfracmajority with bloat control. As observed in Theorem~\ref{thm:+cMajority_lower-bound_constant-stepsize} the \oneonegp with bloat control may never reach the optimum when optimizing an initial tree of size $\sinit < n$. We will deduce that crossover solves this issue and the algorithm reaches the optimum fast. We commence this section by stating the exact formulation of the mentioned result in Theorem~\ref{thm:crossover} followed by an outline of its proof. 
Finally, we show the corresponding result for \supermajority in Theorem~\ref{thm:crossoverSupermajority}.

\begin{theorem}\label{thm:crossover}
	Consider the Concatenation Crossover GP on \geqcmajority or \geqfracmajority with bloat control on the initial tree with size $2\leq n/2 \leq \sinit \leq b \, n$ (for constant $b > 0$). 
	Then there is a constant  $c_{\lambda} >0$ such that for all $c_{\lambda} \log n \leq \lambda \leq n^2$, with probability in $(1- \BigO (n^{-1}))$, the algorithm reaches the optimum after at most $\BigO (n \log^3 (n))$ steps.
\end{theorem}

The following two auxiliary lemmas are used to proof the theorem. First, Lemma~\ref{lem:crossover_time_until_minimal} states the absence of redundant leaves in a \gptree $t$ after the local search with a probability of $1-n^{-5}$. This will be applied after every local search. We observe for two \gptrees $t_1$ and $t_2$ \emph{without redundant leaves}: if $t'$ is the tree resulting from joining $t_1$ and $t_2$, then a variable $i \in [n]$ is expressed in $t'$ if and only if it is expressed in $t_1$ or $t_2$.

Second, Lemma~\ref{lem:crossover_literals_expressed} states that, with a probability of $1-n^{-5}$, each variable $i \in [n]$ is expressed in at least one of $\lambda/2$ trees before the first crossover. Combining both lemmas, for a fixed \gptree $t$ it will suffice to observe the time until $t$ has been joined with at least $\lambda/2$ different trees.

\begin{lemma}\label{lem:crossover_time_until_minimal}
	Consider the \oneonegp with bloat control on either \geqcmajority or \geqfracmajority. For an initial tree with size $2 \leq n/2 \leq \sinit \leq b  n$ (for constant $b > 0$) after $90  \sinit  \log (\sinit)$ iterations, with probability at least $1 - n^{-5}$, the current solution will have no redundant leaves. 
\end{lemma}
\begin{proof}
	Given a \gptree $t$ we define $r(t)$ to be the number of redundant leaves in $t$. Similarly, we define $s(t)$ to be the number of leaves in $t$. By applying the Multiplicative Drift Theorem~\ref{thm:multi_drift_upper_bound} we are going to derive a bound on the expected time $\tau$ until the $\oneonegp$ has for the first time sampled a solution $t$ with $r(t)=0$. Additionally, the theorem will yield that $\tau$ will not be significantly larger than its expectation with high probability.
	
	Given the best-so-far solution $t$ let $t'$ be the next \gptree after one round of mutation and selection in the \oneonegp. In order to apply Theorem~\ref{thm:multi_drift_upper_bound} we need to derive an upper bound on $\Ex{r(t') \mid r(t)}$. First, we observe that due to the bloat control $t'$ cannot have more redundant leaves than $t$. Additionally, the difference will be $1$ if the algorithm chose a redundant leaf and deleted it; we denote this event as $\event{A}$. Let $\overline{\event{A}}$ be the complementary event of $\event{A}$. Due to our observations and the law of total expectation we have
	\begin{align*}
		\Ex{r(t') \mid r(t)} =&\ \Ex{r(t') \mid r(t), \event{A}} \Prob{\event{A}} + \Ex{r(t') \mid r(t), \overline{\event{A}}}\Pr[\overline{\event{A}}] \\
		\leq& \left(1 - \frac{1}{3 s(t)} \right) r(t) \,
	\end{align*}
	since the probability to choose a redundant leaf is $r(t) / s(t)$ followed by the probability of $1/3$ for a deletion. In order to further bound the drift, we need an upper bound on the size during the optimization. Due to the bloat control the worst-case for a growth of the size is an initial tree, where the insertion of $\var_i$ will yield the expression of the variable $i$ for all $i \in [n]$. Therefore, $s(t) \leq n + \sinit \leq 3 \sinit$ and, since this bound is valid for both functions \geqcmajority and \geqfracmajority, we obtain
	\[
		\Ex{r(t') \mid r(t)} \leq \left(1 - \frac{1}{9 \sinit} \right) r(t).
	\]
	Applying the Multiplicative Drift Theorem~\ref{thm:multi_drift_upper_bound} on $r(t) \in \{0\} \cup [1, \sinit]$ with initial tree $t_0$ yields that the time $\tau$ to remove all redundant leaves satisfies
	\[
		\Ex{\tau \mid r(t_{0})} \leq 9 \sinit ( \log (\sinit) + 1).
	\]
	Moreover, we also obtain for every $x > 0$
	\[
		\Prob{\tau > 9 \sinit ( \log (\sinit) + x)} \leq e^{-x}.
	\]
	Therefore, for $x= 9 \log (\sinit)$ we obtain that the solution at iteration $\tau = 90 \sinit \log (\sinit)$ will have no redundant leaves with probability at least $1 - \sinit^{-9} \geq 1-(n/2)^{-9} \geq 1-n^{-5}$.
\end{proof}

\begin{lemma}\label{lem:crossover_literals_expressed}
	Consider the Concatenation Crossover GP on \geqcmajority or \geqfracmajority with bloat control on initial trees with size $2\leq n/2 \leq \sinit \leq b \, n$ (for constant $b > 0$). Then there is a constant  $c_{\lambda} > 0$ such that for all $\lambda \geq c_{\lambda} \log n$, 
	with probability at least $1 - n^{-5}$, each variable will be expressed in at least one of $\lambda/2$ trees before the first crossover.
\end{lemma}
\begin{proof}
	Given a \gptree $t$, for each variable $i \in [n]$ let $\event{A}_i$ be the event that $i$ is expressed in $t$ after initialization. By Lemma~\ref{lem:random-tree} we have $\Prob{\event{A}_i} \geq c^*$ for a constant $c^* >0$. 

	In order to utilize this bound after the Local Search too we observe that the probability $\Prob{\event{A}_i}$ is independent of the choice of $i$. Since the number of expressed variables cannot decrease during Local Search, the same bound of $c^*$ also applies after the Local Search.

	For a variable $i \in [n]$ let $\event{E}_i$ be the event, that $i$ is expressed in at least one of $\lambda /2$ trees after the Local Search. Hence, the complementary event $\overline{\event{E}}_i$ implies that $i$ is not expressed in each of $\lambda/2$ trees. Due to the independence of the trees from each other, 
	\begin{align}\label{eq:crossover_literal_expressed_set} \nonumber
		\Prob{\overline{\event{E}}_i} &\leq \left(1 - \Pr[\event{A}_i]  \right)^{\lambda /2} \leq (1 - c^*)^{\lambda /2}   \\ 
		&\leq e^{- c_{\lambda}  \log(n)  c^* / 2} = n^{- c_{\lambda} c^* / 2} .
	\end{align}

	Let $\event{E}$ be the event that each variable $i \in [n]$ is expressed in at least one of $\lambda/2$ trees after the Local Search. Therefore, the complementary event $\overline{\event{E}}$ implies that any variable $i$ is not expressed in each of $\lambda /2$ trees. $\overline{\event{E}}$ decomposes into the events $\overline{\event{E}}_i$ for $i \in [n]$ and the union bound together with \eqref{eq:crossover_literal_expressed_set} yields
	\begin{align*}
		\Prob{\overline{\event{E}}} = \Prob{\bigcup_{i=1}^{n} \overline{\event{E}}_i} \leq \sum_{i=1}^{n} \Prob{\overline{\event{E}}_i} \leq n \, n^{- c_{\lambda} c^* / 2} = n^{1 - c_{\lambda} c^* / 2}.
	\end{align*}
	
	The only part remaining is to choose $c_{\lambda}$ such that $c_{\lambda} c^* / 2 \geq 6$, which implies the desired result $\Prob{\event{E}} \geq 1 - n^{-5}$.
\end{proof}

Using these two lemmas we can now prove Theorem~\ref{thm:crossover}.

\begin{proof}[Proof of Theorem~\ref{thm:crossover}]
	As explained in the beginning of this section, we will apply Lemma~\ref{lem:crossover_time_until_minimal} after each local search, yielding the absence of redundant leaves. Second, Lemma~\ref{lem:crossover_literals_expressed} yields that each variable $i \in [n]$ is expressed in at least one of $\lambda/2$ trees. Finally, we will study the time until a fixed \gptree $t$ has been joined with at least $\lambda/2$ different trees. Utilizing the Additive Drift Theorem~\ref{thm:additive_drift_upper_bound} with the corresponding Tail Bounds Theorem~\ref{thm:additive_drift_concentration_upper_bound} will yield the desired optimization time as well as the probability for it to hold.
	
	We commence by studying the probability of the event $\event{A}$, that no redundant leaf is left in any of the $\lambda$ trees after local search. Due to Lemma~\ref{lem:crossover_time_until_minimal} we have that for a \gptree $t$  
	$\Prob{r(t)=0 \mbox{ after LS}} \geq 1-n^{-5}$. Thus, we obtain
	$	\Prob{\event{A}} \geq \left(1 - n^{-5} \right)^{\lambda}$
	still approaching $1$ rapidly due to $\lambda = c_{\lambda} \log (n)$.
	Let us assume that the event $\event{A}$ holds and let $t_r$ and $t_j$ be two \gptrees after Local Search, which are going to be joined with crossover together yielding the tree $t_r'$. Due to the absence of redundant leave a variable $i \in [n]$ is expressed in $t_r'$ if and only if it is expressed in $t_r$ or $t_j$. Hence, for a sufficiently large subset $L$ of the $\lambda$ trees, where each variable is expressed in at least one of the trees in $L$, it suffices to study the time until one tree $t$ has been joined with every tree in $L$. The resulting tree $t$ will have every variable expressed and, after Local Search, no redundant leaf left. Lemma~\ref{lem:crossover_literals_expressed} already gives us that any subset $L$ of size at least $\lambda /2$ satisfies the desired property with a high probability.		
	Let $\event{B}$ be the event that no redundant leaf is left in any of the $\lambda$ trees after Local Search and each variable is expressed in at least one of $\lambda /2$ trees. As explained above we obtain
	$\Prob{\event{B}} \geq \left(1-n^{-5} \right) \, \left(1 -n^{-5} \right)^{\lambda}$.
		
	Let us assume that the event $\event{B}$ holds. We will now study the expected time until a fixed \gptree $t$ of the $\lambda$ trees after the initial Local Search has been joined with at least $\lambda /2$ trees, which is essentially a \emph{Coupon Collector} process. This will yield the expected number of crossover cycles $\tau$ until $t$ will be optimal. Additionally, similar to the proof of Lemma~\ref{lem:crossover_time_until_minimal} we obtain that $\tau$ will not be significantly larger than its expectation with high probability. At the end we derive the probability for all the assumed events to hold.
	
	For a \gptree $t$ let $v(t) \in [0, \lambda /2]$ be the number of different \gptrees joined with $t$. We define	$g(t) = \lambda/2 - v(t)$, which measures the remaining \gptrees until $t$ has been joined with at least $\lambda /2$ different trees. Let $t'$ be the offspring of $t$ by joining $t$ with a random \gptree $t_j$ with $j \in [\lambda]$. We obtain for the expected difference between parent and offspring
		$\Ex{g(t) - g(t') \mid g(t)} \geq v(t)/\lambda \geq 1/2 ,$
	because we always have a probability of at least $1/2$ to choose a tree $t_j$, whom we did not join with so far. The Additive Drift Theorem~\ref{thm:additive_drift_upper_bound} yields for the initial tree $t_0$ and the time $\tau$ until we joined $t_0$ with at least $\lambda /2$ different trees
	\begin{align*}
		\Ex{\tau \mid g(t_0)} \leq \lambda .
	\end{align*}
	
	Therefore, we need an expected amount of $\lambda$ cycles until one tree $t$ will be optimal. Moreover, due to the constant step-size of at most $1$ and the finite expected difference of at least $\varepsilon = 1/2$ we are allowed to apply Theorem~\ref{thm:additive_drift_concentration_upper_bound} and obtain for $x \geq \lambda$ that 
	 $ \Prob{\tau \geq x} \leq e^{- x /32 }$.
	We note that the statement of Theorem~\ref{thm:crossover} becomes weaker for larger $c_{\lambda}$. Therefore, for  $\lambda \geq c_{\lambda}\log n$ we may assume $c_{\lambda} \geq 96$ yielding
	\begin{align} \label{eq:crossover_coupon_collector}
		\Prob{\tau \geq 2 \lambda } \leq n^{-6}.
	\end{align}
	What remains is to sum up the runtime for each cycle and derive the probability that $t$ will be optimal.	
	
	Each cycle consists of $\lambda$ crossover operations, which contain a Local Search of fixed time. Due to the absence of redundant leaves we know that every tree $t$ will have size $s(t) \leq c n$ during the crossover. Therefore, said Local Search will be of fixed time at most $90 c n \log (c n)$ and each crossover cycle consists of at most $\lambda \, 90 c n \log (c n)$ steps. Due to $\Prob{\tau \geq 2 \lambda } \leq n^{-6}$ we need $2 \lambda$ cycles with a high probability. Thus, with the probability to be deduced below, the time $\tau'$ until the crossover computes an optimal tree $t$ for the first time is
	\begin{align*}
		\tau' \leq 180 \, c n \, \lambda ^2 \log (c n) \in \BigO (n \log^3 (n)).
	\end{align*}
	This is asymptotically larger than the time for the initial cycle of Local Searches, which consist of at most $\lambda \, 90 \, b n \log (b n) \in \BigO (n \log ^2 (n))$ steps.

	What remains is to calculate the probabilities for the bound to hold. Let $\event{C}$  be the event that $\event{B}$ holds and the Local Search for the tree $t$ expressing all variables yields a tree without redundant leaves. The latter event holds, if every Local Search of each relevant tree up to and including said point yields a tree without redundant leaves. During the crossover there are $\lambda /2$ relevant trees in each crossover cycle. Applying $\Prob{\event{A}} \geq \left(1 - n^{-5} \right)^{\lambda}$ and $\Prob{\event{B}} \geq \left(1-n^{-5} \right) \, \left(1 -n^{-5} \right)^{\lambda}$ we obtain
	\begin{align*}
		\Prob{\event{C}} \geq \Prob{\event{B}} \left(\Prob{\event{A}}^{\lambda/2} \right)^{2 \lambda} \geq  \left(1 - n^{-5} \right)^{\lambda^2 +  \lambda +1}.
	\end{align*}
	In conjunction with the probability that we need at most $2 \lambda$ crosso\-ver cycles
	given by \eqref{eq:crossover_coupon_collector}
	this yields the desired probability of
	$$
		 \left(1 - n^{-6}\right) \Prob{\event{C}} \geq \left(1 - n^{-6}\right) \left(1 - n^{-5} \right)^{\lambda^2 +  \lambda +1} \geq \left(1 - n^{-5} \right)^{(\lambda+1)^2}.
	$$
	Since $\lambda \leq n^2$ this concludes the proof.
\end{proof}

Finally, we turn to \supermajority. For the proof we use a result from the area of rumor spreading relating to the \emph{pull protocol} \cite{Kar-Sch-She-Voc:c:00:RumorSpreading,Mer-Hay-Mat:x:17:RumorSpreading} in order to study the time until every individual of the population has every variable expressed. The idea here is similar to previous proofs with crossover: expressed variables can be collected with crossover. For this purpose we show that the number of $x_i$ in individuals, which have a variable $i$ expressed, is asymptotically larger than the number of $\nonvar_i$ in individuals, which do not have $i$ expressed.

\begin{lemma}[Pull-Protocol~{\cite{Kar-Sch-She-Voc:c:00:RumorSpreading,Mer-Hay-Mat:x:17:RumorSpreading}}]\label{lem:pullProtocol}
Consider the situation where $n$ agents act in rounds in order to gain an information; suppose initially a constant fraction of all agents hold that information. In each round, each uninformed agent selects another agent uniformly at random and is now considered informed if the selected agent is informed. Then, within $\BigO(\log \log n)$ iterations, all agents are informed with probability at least $1-n^{-c}$, for any given $c$.
\end{lemma}

\begin{theorem}\label{thm:crossoverSupermajority}
	Consider the Concatenation Crossover GP without substitutions with bloat control with initial tree size $\sinit = n/2$ on \supermajority. 
	Then there is a constant  $c_{\lambda} >0$ such that, for $\lambda = c_{\lambda} \log n$, each \gptree in the population has all variables expressed after at most $\BigO (n^{1+o(1)})$ steps with probability at least $1- \BigO (n^{-4})$.
\end{theorem}
\begin{proof} 
Fix any given variable $i$. We will show that $i$ is expressed in all individuals after $\log \log n$ rounds of crossover with probability at least $1-n^{-5}$, so that a union bound gives the desired probability (leaving the bound on the number of steps to be shown). In the following we reason for $n$ large enough.

 Using a Chernoff bound we see that, with probability at least $1-n^{-6}$, the maximum number of $\nonvar_i$ in a \gptree after initialization is at most $20\log n/ \log \log n$. Note that local search cannot add any $\nonvar_i$, since we use bloat control and do not have substitutions. 
Thus, by induction and using that all crossover steps are based on the old population, after $k \geq 0$ rounds of crossover every individual of the population has at most $2^{k} \cdot 20\log n/ \log \log n$ occurrences of $\nonvar_i$.

Consider now any individual where $i$ is \emph{expressed} and where at least $n/c$ variables are expressed, for some $c$; let $s$ be the size of that individual. Note that $s \geq n/c$. Another application of the Chernoff bound give that, after local search is applied to this individual for $90 s \log s$ iterations, we have at least $s \log (s)/n $ many $\var_i$ with probability $1-n^{-6}$, while we get for $n$ sufficiently large that the tree grows by at least $2s \log(s)/ \log \log n \geq s\log n/\log \log n$ with probability at least $1-n^{-6}$ and by at most $90s \log s$ with certainty. Induction gives that, after $k$ rounds of crossover, with probability at least $1-2kn^{-6}$, the total size of any such \gptree is at least $n (\log n / \log \log n)^k$ and the number of $\var_i$ is at least $(\log n)^{k+1}$. 

Thus, whenever we mate two individuals where one has $i$ expressed and the other does not, the offspring has $i$ expressed: the number of $\var_i$ in the individual with $i$ expressed is at least $(\log n)^{k+1}$ and thus larger than the number of $\nonvar_i$ in any mate, which is at most $2^{k} \cdot 20\log n/ \log \log n$.

Initially a constant fraction of all individuals had $i$ expressed with sufficiently high probability (see Lemma~\ref{lem:random-tree}). Each round of crossover works like the pull protocol in rumor spreading to get the variable $i$ expressed, so that Lemma~\ref{lem:pullProtocol} gives that, after $\BigO(\log \log n)$ rounds of crossover, all individuals have $i$ expressed with probability at least $1-n^{-6}$.

While the size of a tree is at most $n^2$, it grows at most by a factor of $2\log n$, leading to a tree size of $n (\log n)^{O(\log \log n)} \leq n^{1+o(1)}$ after $\BigO(\log \log n)$ iterations. 

In order to compute the total run time of all local searchers, note that the growth of individuals is faster than exponential. Since we only desire an asymptotic bound, we can ignore all but the last summand. We apply local search to each of logarithmically many individuals, which increases the run time by another logarithmic factor, still yielding a bound of $n^{1+o(1)}$. A union bound over all failure probabilities gives the desired result.
\end{proof}

\section{Experiments}
\label{sec:experiments}
This section is dedicated to complementing our theoretical results with experimental justification for the otherwise open cells of Table~\ref{table:overviewResults}, i.e.\ for the \oneonegp without bloat control on \geqcmajority and \geqfracmajority.

\begin{figure}
\begin{minipage}{0.5\textwidth}
\resizebox{\textwidth}{!}{
\begin{tikzpicture}
\begin{axis}[
		legend style={
		at={(0.05,0.7)},
		anchor=west},
		xlabel= {n, number of variables},
		ylabel= {number of evaluations},
		boxplot/draw direction=y,
		baseline,
		xtick = {100, 300,500,700,900},
		xticklabels = {100, 300,500,700,900}
	]
	\addplot[blue, domain = 100.1:100.2]{10000*x};
	\addlegendentry{\oneonegp with bloat control};
	\addplot[lime!70!black, domain = 100.1:1000.1]{28*(x)^(1)*ln(x)};
	\addlegendentry{\oneonegp without bloat control};
	
\addplot+ [color=blue,solid,boxplot prepared = {box extend=50, draw position= 100, lower whisker = 1000001, lower quartile = 1000001, median =1000001, upper quartile = 1000001, upper whisker = 1000001},]coordinates{};

\addplot+ [color=blue,solid,boxplot prepared = {box extend=50, draw position= 200, lower whisker = 1000001, lower quartile = 1000001, median =1000001, upper quartile = 1000001, upper whisker = 1000001},]coordinates{}; 

\addplot+ [color=blue,solid,boxplot prepared = {box extend=50, draw position= 300, lower whisker = 1000001, lower quartile = 1000001, median =1000001, upper quartile = 1000001, upper whisker = 1000001},]coordinates{}; 

\addplot+ [color=blue,solid,boxplot prepared = {box extend=50, draw position= 400, lower whisker = 1000001, lower quartile = 1000001, median =1000001, upper quartile = 1000001, upper whisker = 1000001},]coordinates{}; 

\addplot+ [color=blue,solid,boxplot prepared = {box extend=50, draw position= 500, lower whisker = 1000001, lower quartile = 1000001, median =1000001, upper quartile = 1000001, upper whisker = 1000001},]coordinates{}; 

\addplot+ [color=blue,solid,boxplot prepared = {box extend=50, draw position= 600, lower whisker = 1000001, lower quartile = 1000001, median =1000001, upper quartile = 1000001, upper whisker = 1000001},]coordinates{}; 
\addplot+ [color=blue,solid,boxplot prepared = {box extend=50, draw position= 700, lower whisker = 1000001, lower quartile = 1000001, median =1000001, upper quartile = 1000001, upper whisker = 1000001},]coordinates{}; 
\addplot+ [color=blue,solid,boxplot prepared = {box extend=50, draw position= 800, lower whisker = 1000001, lower quartile = 1000001, median =1000001, upper quartile = 1000001, upper whisker = 1000001},]coordinates{}; 
\addplot+ [color=blue,solid,boxplot prepared = {box extend=50, draw position= 900, lower whisker = 1000001, lower quartile = 1000001, median =1000001, upper quartile = 1000001, upper whisker = 1000001},]coordinates{}; 
\addplot+ [color=blue,solid,boxplot prepared = {box extend=50, draw position= 1000, lower whisker = 1000001, lower quartile = 1000001, median =1000001, upper quartile = 1000001, upper whisker = 1000001},]coordinates{}; 


\addplot+ [color=lime!70!black,solid,boxplot prepared = {box extend=50, draw position= 100, lower whisker = 7226, lower quartile = 10476, median =12390, upper quartile = 14915, upper whisker = 20001},]coordinates{}; 
\addplot+ [color = lime!70!black,solid,boxplot prepared = {box extend=50, draw position= 200, lower whisker = 19115, lower quartile = 27194, median =32387, upper quartile = 37273, upper whisker =65341},]coordinates{}; 
\addplot+ [color = lime!70!black,solid,boxplot prepared = {box extend=50, draw position= 300, lower whisker = 28620, lower quartile = 41715, median =50724, upper quartile = 58971, upper whisker =84148},]coordinates{}; 
\addplot+ [color = lime!70!black,solid,boxplot prepared = {box extend=50, draw position= 400, lower whisker = 42494, lower quartile = 59211, median =66442, upper quartile = 76374, upper whisker =115541},]coordinates{}; 
\addplot+ [color = lime!70!black,solid,boxplot prepared = {box extend=50, draw position= 500, lower whisker = 53583, lower quartile = 78364, median =88097, upper quartile = 104832, upper whisker =157985},]coordinates{}; 
\addplot+ [color = lime!70!black,solid,boxplot prepared = {box extend=50, draw position= 600, lower whisker = 75113, lower quartile = 96844, median =110481, upper quartile = 122778, upper whisker =196772},]coordinates{}; 
\addplot+ [color = lime!70!black,solid,boxplot prepared = {box extend=50, draw position= 700, lower whisker = 91755, lower quartile = 113927, median =129548, upper quartile = 158948, upper whisker =254820},]coordinates{}; 
\addplot+ [color = lime!70!black,solid,boxplot prepared = {box extend=50, draw position= 800, lower whisker = 94772, lower quartile = 134642, median =153157, upper quartile = 177321, upper whisker =293312},]coordinates{}; 
\addplot+ [color = lime!70!black,solid,boxplot prepared = {box extend=50, draw position= 900, lower whisker = 115669, lower quartile = 150991, median =166883, upper quartile = 208963, upper whisker =332226},]coordinates{}; 
\addplot+ [color = lime!70!black,solid,boxplot prepared = {box extend=50, draw position= 1000, lower whisker = 132451, lower quartile = 179094, median =200911, upper quartile = 226707, upper whisker =364773},]coordinates{}; 

\end{axis}
\end{tikzpicture}
}
\end{minipage}
\hfill
\begin{minipage}{0.49\textwidth}
\resizebox{\textwidth}{!}{
\begin{tikzpicture}
\begin{axis}[
		legend style={
		at={(0.05,0.85)},
		anchor=west},
		xlabel= {n, number of variables},
		ylabel= {number of evaluations},
		boxplot/draw direction=y,
		baseline,
		xtick = {100, 300,500,700,900},
		xticklabels = {100, 300,500,700,900}
	]

	\addplot[lime!70!black, domain = 100.1:1000.1]{32*(x)^(1)*ln(x)};
	\addlegendentry{\oneonegp without bloat control};
	\addplot[blue, domain = 100.1:1000.1]{9*(x)^(1)*ln(x)};
	\addlegendentry{\oneonegp with bloat control};

\addplot+ [color = blue,solid,boxplot prepared = {box extend=50, draw position= 100, lower whisker = 3819, lower quartile = 4424, median =4946, upper quartile = 5503, upper whisker = 9502},]coordinates{}; 
\addplot+ [color = blue,solid,boxplot prepared = {box extend=50, draw position= 200, lower whisker = 7871, lower quartile = 9563, median =10585, upper quartile = 11302, upper whisker =16996},]coordinates{}; 
\addplot+ [color = blue,solid,boxplot prepared = {box extend=50, draw position= 300, lower whisker = 12469, lower quartile = 15021, median =16353, upper quartile = 17976, upper whisker =23565},]coordinates{}; 
\addplot+ [color = blue,solid,boxplot prepared = {box extend=50, draw position= 400, lower whisker = 17995, lower quartile = 21214, median =22823, upper quartile = 25096, upper whisker =34149},]coordinates{}; 
\addplot+ [color = blue,solid,boxplot prepared = {box extend=50, draw position= 500, lower whisker = 22449, lower quartile = 27092, median =29816, upper quartile = 32147, upper whisker =39841},]coordinates{}; 
\addplot+ [color = blue,solid,boxplot prepared = {box extend=50, draw position= 600, lower whisker = 28263, lower quartile = 33162, median =35671, upper quartile = 38908, upper whisker =51037},]coordinates{}; 
\addplot+ [color = blue,solid,boxplot prepared = {box extend=50, draw position= 700, lower whisker = 32302, lower quartile = 38409, median =43161, upper quartile = 46267, upper whisker =63029},]coordinates{}; 
\addplot+ [color = blue,solid,boxplot prepared = {box extend=50, draw position= 800, lower whisker = 41008, lower quartile = 45666, median =49784, upper quartile = 53059, upper whisker =67676},]coordinates{}; 
\addplot+ [color = blue,solid,boxplot prepared = {box extend=50, draw position= 900, lower whisker = 43791, lower quartile = 51344, median =54299, upper quartile = 59124, upper whisker =94864},]coordinates{}; 
\addplot+ [color = blue,solid,boxplot prepared = {box extend=50, draw position= 1000, lower whisker = 50512, lower quartile = 57731, median =60696, upper quartile = 63959, upper whisker =92700},]coordinates{};

\addplot+ [color=lime!70!black,solid,boxplot prepared = {box extend=50, draw position= 100, lower whisker = 10807, lower quartile = 13909, median =16582, upper quartile = 19007, upper whisker = 24940},]coordinates{}; 
\addplot+ [color = lime!70!black,solid,boxplot prepared = {box extend=50, draw position= 200, lower whisker = 27486, lower quartile = 34644, median =38009, upper quartile = 41792, upper whisker =51561},]coordinates{}; 

\addplot+ [color = lime!70!black,solid,boxplot prepared = {box extend=50, draw position= 300, lower whisker = 38279, lower quartile = 52286, median =56758, upper quartile = 64294, upper whisker =85663},]coordinates{}; 
\addplot+ [color = lime!70!black,solid,boxplot prepared = {box extend=50, draw position= 400, lower whisker = 62485, lower quartile = 74318, median =79409, upper quartile = 86752, upper whisker =126599},]coordinates{}; 
\addplot+ [color = lime!70!black,solid,boxplot prepared = {box extend=50, draw position= 500, lower whisker = 76350, lower quartile = 96858, median =107371, upper quartile = 116230, upper whisker =149708},]coordinates{}; 
\addplot+ [color = lime!70!black,solid,boxplot prepared = {box extend=50, draw position= 600, lower whisker = 95389, lower quartile = 114394, median =126019, upper quartile = 139268, upper whisker =188158},]coordinates{}; 
\addplot+ [color = lime!70!black,solid,boxplot prepared = {box extend=50, draw position= 700, lower whisker = 115990, lower quartile = 136227, median =146858, upper quartile = 161496, upper whisker =211307},]coordinates{}; 
\addplot+ [color = lime!70!black,solid,boxplot prepared = {box extend=50, draw position= 800, lower whisker = 140852, lower quartile = 159022, median =172195, upper quartile = 190843, upper whisker =236012},]coordinates{}; 
\addplot+ [color = lime!70!black,solid,boxplot prepared = {box extend=50, draw position= 900, lower whisker = 151745, lower quartile = 171926, median =191369, upper quartile = 212662, upper whisker =265133},]coordinates{}; 
\addplot+ [color = lime!70!black,solid,boxplot prepared = {box extend=50, draw position= 1000, lower whisker = 163406, lower quartile = 200224, median =217875, upper quartile = 236190, upper whisker =299607},]coordinates{}; 

\end{axis}
\end{tikzpicture}
}
\end{minipage}
\caption{Number of evaluations required by the \oneonegp over 100 runs for each $n$ with the initial tree size $\tinit =10n$ until all variables are expressed or the time limit, equal to 1000000 evaluations, is reached. The left figure shows the experimental results for $\geqcmajority$ with $c=2$; the solid line is $28n\log n$. On the right figure is shown $\geqfracmajority$; the blue solid line is $9n\log n$, the green solid line is $32n\log n$.  } 
\label{fig:2maj_plot}
\end{figure}
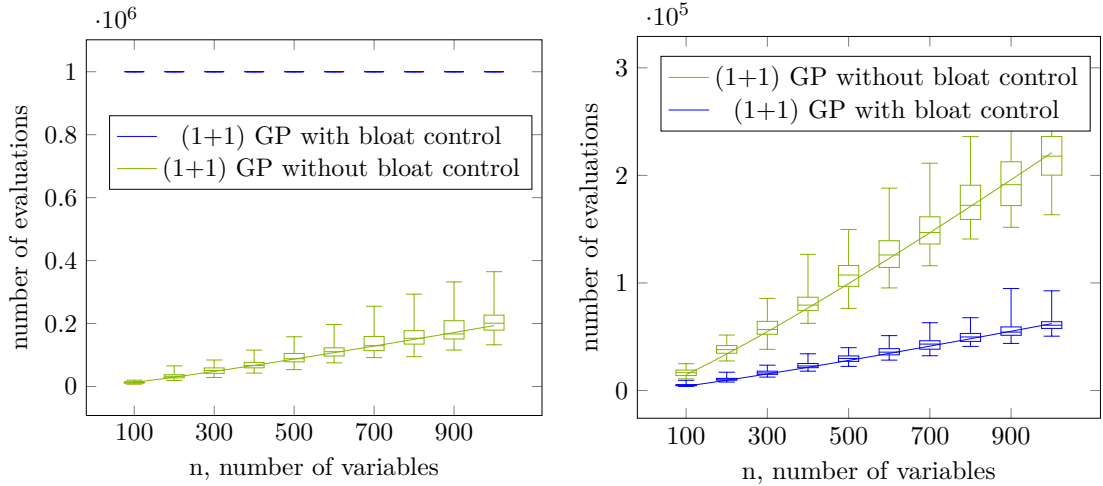

All experimental results shown in Figure~\ref{fig:2maj_plot} 
are box-and-whiskers plots, where lower and upper whiskers are the minimal and maximal number of \emph{fitness evaluations} the algorithm required over 100 runs until all variables are expressed or the time limit of 1000000 evaluations is reached. The middle lines in each box are the median values (the second quartile), the bottom and top of the boxes are the first and third quartiles. Note that all experiments are platform independent since we count number of fitness evaluations independently of real time. The solid lines in the plots allow to estimate the asymptotic run time of the \oneonegp.

The left hand side of Figure~\ref{fig:2maj_plot} concerns \geqcmajority and shows that the \oneonegp with bloat control always fails (corresponding to Theorem~\ref{thm:+cMajority_lower-bound_constant-stepsize}). We used the \oneonegp with $\tinit=10n,$ $c=2$ and $n$ as indicated along the x-axis. It is easy to see that bloat control leads the algorithm to local optima and does not allow to leave it, whereas the \oneonegp \emph{without} bloat control finds an optimum in a reasonable number of evaluations. Due to time and computational restrictions the constant $c$ was chosen equal to 2. For larger $c$ the run time of the algorithm goes up significantly, but a similar pattern is visible.

The right hand side of Figure~\ref{fig:2maj_plot} shows the results of \oneonegp on \geqfracmajority, using $\tinit=10n$. One can see that bloat control is more efficient in comparison with the \oneonegp without bloat control. The set of median values is well-approximated by $w \cdot n\log n$ for a constant $w$, which leads us to the conjecture that the algorithm's run time is $\BigO(n\log n)$. 
We did not analyze the influence of $\tinit$, but it might be significant especially for \geqfracmajority without bloat control.

\section{Conclusion}
\label{sec:conclusion}
We defined three variants of the \majority problem in order to introduce some fitness plateaus that are difficult to cross. The \geqcmajority allows for progress at the end of the plateau with large representation; in this sense, bloat is necessary for progress. On the other hand, for \geqfracmajority, progress can be made at the end of the plateau with small representation, so that bloat control guides the search to the fruitful part of the search space. We also considered \supermajority which exemplifies fitness functions where bloat is inherent due to the possibility of small improvements by adding an increasing amount of nodes to the \gptree. In this case we showed that not employing bloat control leads to inefficient optimization.

In order to obtain results somewhat closer to practically relevant GP we turned to crossover and showed how a Concatenation Crossover GP can efficiently optimize all three considered test functions.

For future work it might be interesting to analyze the effect of other crossover operators. In order to obtain a better understanding of such other operators, additional test functions might be necessary making essential use of the tree structure (all our test functions might as well use lists or even multisets of the leaves as representations). Such test functions should not be too complex, which would hinder a theoretical analysis, but still embody a structure frequently found in GP, so as to inform about relevant application areas. The search for such test functions remains a central open problem of the theory of GP.

\bibliographystyle{plain}
\bibliography{ms}

\clearpage

\end{document}